\documentclass[acmsmall,authorversion]{acmart}

\usepackage{booktabs}
\usepackage{nicefrac}
\usepackage{siunitx}
\usepackage{enumitem}
\usepackage{subcaption}
\usepackage{tikz}
\usetikzlibrary{positioning}
\usepackage{algorithm}
\usepackage{algorithmic}
\usepackage{tcolorbox}
\usepackage{mdframed}
\usepackage{cleveref}

\newtheorem{theorem}{Theorem}

\theoremstyle{definition}
\newtheorem{definition}{Definition}
\newtheorem{remark}{Remark}

\newcommand{\R}{\mathbb{R}}
\newcommand{\E}{\mathbb{E}}

\renewcommand\footnotetextcopyrightpermission[1]{}
\setcopyright{none}
\acmConference{}{}{}
\acmYear{}
\acmDOI{}
\acmISBN{}

\title{The Bayesian Geometry of Transformer Attention}
\subtitle{Paper I of the Bayesian Attention Trilogy}

\author{Naman Agarwal}
\affiliation{
  \institution{Dream Sports}
  \city{New York}
  \state{NY}
  \country{USA}
}
\email{naman33k@gmail.com}
\authornote{Currently at Google DeepMind. Work performed while at Dream Sports.}

\author{Siddhartha R. Dalal}
\affiliation{
  \institution{Columbia University}
  \department{School of Professional Studies and Department of Statistics}
  \city{New York}
  \state{NY}
  \country{USA}
}
\email{sd2803@columbia.edu}

\author{Vishal Misra}
\affiliation{
  \institution{Columbia University}
  \department{Department of Computer Science}
  \city{New York}
  \state{NY}
  \country{USA}
}
\email{vishal.misra@columbia.edu}

\authorsaddresses{Authors' emails: naman33k@gmail.com, sd2803@columbia.edu, vishal.misra@columbia.edu. Authors are listed in alphabetical order.}

\begin{document}

\begin{abstract}
Modern sequence models often appear to behave as Bayesian learners, but it remains unclear whether this reflects genuine probabilistic inference or task-specific heuristics. We introduce \emph{Bayesian wind tunnels}---controlled environments where the true posterior is known in closed form and memorization is provably impossible---to resolve this question empirically. In these settings, small transformers reproduce exact Bayesian posteriors for filtering and hypothesis elimination with \mbox{$10^{-3}$--$10^{-4}$} bit accuracy, while capacity-matched MLPs fail by orders of magnitude.

To understand which architectural ingredients enable exact inference, we decompose Bayesian computation into three \emph{inference primitives}: (i)~\textbf{belief accumulation}---integrating evidence into a running posterior; (ii)~\textbf{belief transport}---propagating beliefs forward through stochastic dynamics; and (iii)~\textbf{random-access binding}---retrieving stored hypotheses by content rather than position. Different tasks demand different subsets of these primitives, and different architectures can realize different subsets.

Comparing Transformers, Mamba, LSTMs, and MLPs across bijection learning, HMM filtering, and associative recall, we find that Transformers realize all three primitives; Mamba realizes accumulation and transport but struggles with random-access binding; LSTMs realize only accumulation (of static sufficient statistics); and MLPs realize none. Geometric diagnostics reveal orthogonal key bases, low-dimensional value manifolds parameterized by posterior entropy, and---in Mamba---five discrete clusters corresponding to HMM hidden states.

These results demonstrate that Bayesian computation is not monolithic: its realizability depends on the inference primitives a task demands and the architectural mechanisms available to implement them. Bayesian wind tunnels provide a foundation for mechanistically connecting small, verifiable systems to reasoning phenomena observed in large language models.
\end{abstract}

\maketitle

\section{Introduction}
\label{sec:intro}

Can transformers perform exact Bayesian inference---filtering and hypothesis elimination---or do they merely approximate it through pattern matching? Natural language offers no ground-truth posterior against which to verify predictions, and modern LLMs are too large and too entangled with their data to separate genuine probabilistic computation from memorization. Even when models \emph{behave} Bayesianly, there is no direct way to confirm that the internal computation matches Bayes' rule.

\vspace{4pt}
\noindent\textbf{Our approach.}
We replace unverifiable natural data with \emph{Bayesian wind tunnels}: controlled prediction tasks where
\begin{enumerate}[itemsep=2pt, topsep=2pt]
    \item the \emph{analytic posterior} is known exactly at each step,
    \item the \emph{hypothesis space} is so large that memorization is computationally infeasible,
    \item in-context prediction requires \emph{genuine probabilistic inference}.
\end{enumerate}
This converts a qualitative question (``does it do Bayes?'') into a quantitative test: does the model's predictive entropy match the analytic posterior entropy position by position?

\vspace{4pt}
\noindent\textbf{Four wind tunnels.}
We study four settings:
\begin{itemize}[itemsep=2pt, topsep=2pt]
    \item \textbf{Bijection learning:} a discrete hypothesis-elimination problem with a closed-form posterior.
    \item \textbf{Hidden Markov Models (HMMs):} a sequential, stochastic inference problem requiring recursive updates.
    \item \textbf{Bayesian regression:} a continuous inference problem with closed-form Gaussian posterior over linear weights.
    \item \textbf{Associative recall:} a content-based retrieval task testing the binding primitive.
\end{itemize}

To understand which architectural ingredients enable Bayesian inference, we decompose it into three \emph{inference primitives}:
\begin{enumerate}[itemsep=2pt]
    \item \textbf{Belief accumulation:} integrating evidence into a running posterior (e.g., updating $P(\theta \mid x_{1:t})$ as observations arrive).
    \item \textbf{Belief transport:} propagating beliefs forward through stochastic dynamics (e.g., HMM filtering where hidden states evolve).
    \item \textbf{Random-access binding:} retrieving stored hypotheses by content rather than position (e.g., recalling a target given a probe cue).
\end{enumerate}
Different tasks demand different subsets of these primitives, and different architectures can realize different subsets. We test four architectures: Transformers, Mamba (a selective state-space model), LSTMs, and MLPs. \Cref{fig:primitives} summarizes which architectures realize which primitives.

\begin{figure}[t]
    \centering
    \includegraphics[width=0.85\columnwidth]{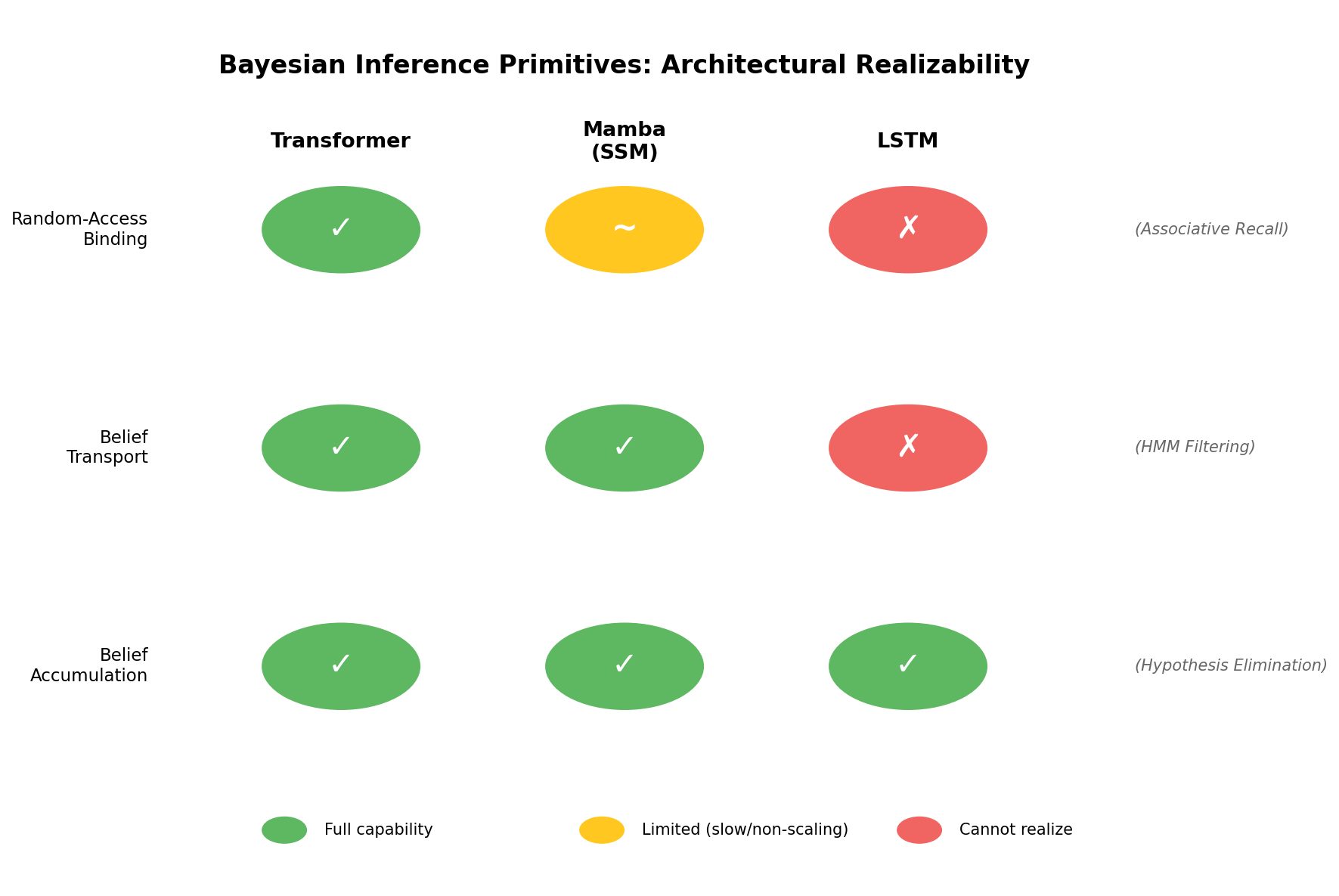}
    \caption{Inference primitives taxonomy. Each row is a primitive; each column is an architecture. Transformers realize all three primitives. Mamba realizes accumulation and transport but struggles with random-access binding. LSTMs realize only accumulation of static sufficient statistics. MLPs realize none.}
    \label{fig:primitives}
\end{figure}

\vspace{4pt}
\noindent\textbf{Findings.}
Transformers realize all three primitives and succeed on all tasks. Mamba realizes accumulation and transport---achieving state-of-the-art on HMM filtering---but struggles with random-access binding in associative recall. LSTMs realize only accumulation of \emph{static} sufficient statistics: they succeed on belief revision (where the statistic is fixed-dimensional) but fail when the sufficient statistic itself evolves under dynamics or must be indexed by content. MLPs realize none and fail uniformly.

Geometric diagnostics reveal orthogonal key axes, progressive alignment, and a low-dimensional value manifold parameterized by entropy. On HMM tracking, Mamba's final layer organizes into five clusters---one per hidden state---showing that the model discovers the corner geometry of the belief simplex.

\vspace{4pt}
\noindent\textbf{Contribution.}
This paper provides the first empirical proof that neural sequence models can realize exact Bayesian posteriors, introduces Bayesian wind tunnels as a tool for probing algorithmic reasoning in verifiable settings, and establishes a taxonomy of inference primitives that explains which architectures succeed on which tasks. The primitives framework unifies prior observations: content-based routing enables accumulation and transport; attention's random-access capability additionally enables binding.

\vspace{4pt}
\noindent\textbf{Clarification on ``Bayesian inference.''}
We do not claim a Bayesian posterior over network weights; we show that the learned predictor implements the \emph{Bayesian posterior predictive over latent task variables}---the filtering posterior over hidden states in HMMs, or the elimination posterior over bijections. This is a statement about the input-output function the transformer computes, not about weight-space uncertainty. Specifically, we demonstrate exact Bayesian inference in tasks whose posteriors factorize sequentially over the input---filtering and elimination, not general Bayesian model selection.

\vspace{4pt}
\noindent\textbf{Roadmap: this paper as Lemma 1.}
This paper establishes the \emph{existence and internal geometry} of exact Bayesian inference in transformers under verifiable conditions. It is the first of three papers forming a unified argument: Paper II shows that this geometry arises generically from gradient dynamics under cross-entropy training, explaining \emph{why} transformers learn Bayesian structure. Paper III shows how these primitives compose in partially observed settings closer to natural language. Together, the trilogy characterizes when, why, and how neural sequence models implement probabilistic reasoning.

\vspace{6pt}
\begin{mdframed}[linewidth=0.5pt, innertopmargin=8pt, innerbottommargin=8pt]
\noindent\textbf{Structural Theorem (Inference Primitives and Architectural Realizability).}
\emph{Bayesian inference in sequential prediction decomposes into three primitives:}
\begin{enumerate}[itemsep=1pt, topsep=2pt, leftmargin=1.5em]
    \item \emph{belief accumulation over fixed latent hypotheses,}
    \item \emph{belief transport under latent state dynamics, and}
    \item \emph{random-access binding between beliefs and past observations.}
\end{enumerate}
\emph{Neural sequence architectures differ not in whether they can approximate Bayesian inference, but in which primitives they can realize:}
\begin{itemize}[itemsep=1pt, topsep=2pt, leftmargin=1.5em]
    \item \emph{Recurrent architectures (LSTMs) implement accumulation of static sufficient statistics, but fail when inference requires transport of probability mass under nontrivial dynamics.}
    \item \emph{State-space models (Mamba) implement accumulation and transport, but fail when inference requires random-access binding to arbitrary past observations.}
    \item \emph{Attention-based transformers implement all three primitives by externalizing belief as a geometric, addressable representation rather than compressing it into fixed-size state.}
\end{itemize}
\emph{The dominance of transformers in reasoning tasks arises not from scale alone, but from primitive completeness: they are the minimal architecture realizing the full set of inference primitives.}
\end{mdframed}
\vspace{2pt}

\section{Theoretical Framework: Cross-Entropy and Bayesian Inference}
\label{sec:theory}

Cross-entropy training on contextual prediction tasks has a well-known population optimum: the Bayesian posterior predictive distribution. This section formalizes that connection. The theory establishes \emph{what} the learned function should be in the infinite-data, infinite-capacity limit; the empirical sections evaluate \emph{which architectures can approximate it} in finite settings.

\subsection{Setup}

Consider a family of tasks indexed by a latent parameter $\theta \sim \pi(\theta)$.
For each task:
\begin{itemize}[leftmargin=*, itemsep=2pt]
    \item inputs $x$ are drawn from some distribution (possibly adversarial or chosen by the experimenter),
    \item labels are drawn according to $y \sim p(y \mid x, \theta)$,
    \item the model observes a context $c = \{(x_i, y_i)\}_{i=1}^{k}$ and must predict $y$ for a new query input.
\end{itemize}

We train a model $q(y \mid x, c)$ by minimizing population cross-entropy:
\begin{equation}
\label{eq:ce}
\mathcal{L}(q)
~=~ \E_{\theta \sim \pi} \E_{c,(x,y)\sim p(\cdot|\theta)}
\left[-\log q(y \mid x, c) \right].
\end{equation}

\subsection{Cross-entropy minimizes to the Bayesian posterior predictive}

\begin{theorem}[Population optimum of cross-entropy]
\label{thm:bayes_ce}
The minimizer of \eqref{eq:ce} is the Bayesian posterior predictive distribution
\begin{equation}
q^\star(y \mid x, c)
~=~
\int p(y \mid x, \theta)\,
p(\theta \mid c)\, d\theta,
\end{equation}
where
\begin{equation}
p(\theta \mid c)
~\propto~
\pi(\theta) \prod_{(x_i,y_i)\in c} p(y_i \mid x_i, \theta).
\end{equation}
\end{theorem}

\begin{proof}
Fixing $(x,c)$ and taking expectation over $y\sim p(\cdot\mid x,c)$,
\[
\arg\min_q \E[-\log q(y\mid x,c)]
~=~
p(y\mid x,c),
\]
which equals $\int p(y\mid x,\theta)p(\theta\mid c)d\theta$ by Bayes' rule and the factorization $(y\!\perp\! c)\mid(x,\theta)$.
\end{proof}

\begin{remark}
This result is \emph{architecture-agnostic}: it defines the Bayes-optimal function but not whether any particular architecture can represent or learn it. Our experiments address this realizability question directly.
\end{remark}

\subsection{Application to the bijection wind tunnel}
\label{sec:theory-bijection}

In the bijection task, each $\theta$ is a bijection $\pi:\{1,\dots,V\}\to\{1,\dots,V\}$.  
A training sequence reveals $k-1$ input--output pairs. Let $\mathcal{O}_{k-1}$ be the set of outputs already observed. Because each input appears at most once per sequence, the current query $x_k$ has never been seen before, so Bayes' rule reduces to:
\begin{equation}
p(\pi(x_k)=y \mid c)
=
\begin{cases}
\dfrac{1}{V-k+1}, & y\notin \mathcal{O}_{k-1},\\
0, &\text{otherwise}.
\end{cases}
\end{equation}

Hence the analytic posterior entropy is
\begin{equation}
H_{\text{Bayes}}(k)=\log_2(V-k+1),
\label{eq:bayes-bijection-entropy}
\end{equation}
producing a monotone staircase that shrinks by one bit whenever a mapping is revealed.

This closed-form posterior allows direct, position-by-position comparison between model entropy and Bayesian entropy; memorization is computationally infeasible because the hypothesis space size $V!$ is enormous.

\subsection{Application to the HMM wind tunnel}
\label{sec:theory-hmm}

In the HMM task, each $\theta$ consists of:
\begin{itemize}[leftmargin=*, itemsep=2pt]
    \item a transition matrix $T \in \R^{S\times S}$,
    \item an emission matrix $E \in \R^{S\times V}$,
    \item an initial state distribution $\pi_0$.
\end{itemize}
After observing $o_{1:t}$, the true Bayesian posterior over hidden states is given by the forward algorithm:
\begin{equation}
\alpha_t(s)
~=~
p(s_t=s \mid o_{1:t})
~=~
\frac{
E(o_t \mid s)
\sum_{s'} T(s \mid s') \alpha_{t-1}(s')
}{
\sum_{s''} E(o_t \mid s'') \sum_{s'} T(s'' \mid s') \alpha_{t-1}(s')
}.
\label{eq:hmm-forward}
\end{equation}

Since models predict the next observation $o_{t+1}$ rather than hidden states, we evaluate the \emph{predictive entropy}:
\begin{equation}
H_{\text{Bayes}}(t)
~=~
-\sum_{o} p(o_{t+1} \mid o_{1:t}) \log_2 p(o_{t+1} \mid o_{1:t}),
\label{eq:hmm-entropy}
\end{equation}
where $p(o_{t+1} \mid o_{1:t}) = \sum_{s,s'} \alpha_t(s) T(s' \mid s) E(o_{t+1} \mid s')$.

Because every training sequence is generated from a freshly sampled $(T,E)$, the hypothesis space is massive and memorization is computationally infeasible. The model must learn to (i) parse the header encoding $T$ and $E$, and (ii) implement a recursive Bayesian update.

\subsection{Application to the regression wind tunnel}
\label{sec:theory-regression}

In the regression task, each $\theta$ is a weight vector $w \in \mathbb{R}^d$ with prior $w \sim \mathcal{N}(0, I)$. Given context observations $(x_i, y_i)$ where $y_i = w^\top x_i + \epsilon_i$ with $\epsilon_i \sim \mathcal{N}(0, \sigma^2)$, the posterior over $w$ is Gaussian with closed-form mean and covariance. The posterior predictive at any query point $x$ is also Gaussian, enabling exact computation of the Bayesian predictive distribution for comparison.

\subsection{Application to the associative recall wind tunnel}
\label{sec:theory-recall}

In the associative recall task, each sequence contains $N$ cue--target pairs $(c_i, t_i)$ followed by probe cues. Unlike the other tasks which test belief accumulation or transport, associative recall tests the \emph{binding} primitive: the model must store all pairs and retrieve by content when the probe arrives. Success requires content-based routing---identifying which stored pair matches the probe---rather than sequential Bayesian updating.

\subsection{Implications for model evaluation}

The theoretical results above imply a practical diagnostic:
\emph{a model that achieves the correct posterior entropy at every position is functionally Bayesian---it produces predictions with the same uncertainty profile as the exact posterior.} Combined with the cross-entropy training objective (whose unique population minimizer is the Bayesian posterior predictive), low entropy calibration error provides strong evidence for Bayesian computation.

\paragraph{Beyond entropy: full distributional verification.}
Entropy-matching is necessary but not sufficient for distributional convergence in general, since different distributions can share the same entropy. In our wind tunnels, however, entropy serves as a \emph{diagnostic sufficient statistic}: the structured nature of the tasks (discrete elimination over bijections, recursive filtering over HMM states) strongly constrains the space of distributions achieving a given entropy. To confirm full distributional equivalence, we directly verify via KL divergence that our trained transformers match the complete Bayesian posterior (\Cref{tab:kl_tvd}). Throughout this paper we report entropy MAE as our primary metric because it provides an interpretable, bit-level measure that generalizes across tasks.

\begin{table}[h]
\centering
\caption{Full distributional verification: KL divergence (nats) and total variation distance between transformer predictions and exact Bayesian posteriors. Bijection results are stratified by key type: ``new'' keys (first occurrence) vs.\ ``repeated'' keys (seen before in the sequence). Values are averaged over 2{,}000 held-out sequences.}
\label{tab:kl_tvd}
\begin{small}
\begin{tabular}{lcc}
\toprule
\textbf{Condition} & \textbf{KL}$(\text{model}\|\text{Bayes})$ & \textbf{TVD} \\
\midrule
\multicolumn{3}{l}{\emph{Bijection task}} \\
\quad New keys & $8.2 \times 10^{-4}$ & 0.027 \\
\quad Repeated keys & $< 10^{-6}$ & $< 10^{-4}$ \\
\midrule
\multicolumn{3}{l}{\emph{HMM task ($K=20$)}} \\
\quad All positions & $1.1 \times 10^{-4}$ & 0.018 \\
\bottomrule
\end{tabular}
\end{small}
\end{table}

Evaluating the \textbf{entropy calibration error}
\begin{equation}
\text{MAE}
=
\frac{1}{L}\sum_k
\left|H_{\text{model}}(k) - H_{\text{Bayes}}(k)\right|
\label{eq:mae}
\end{equation}
therefore provides a direct, bit-level measure of Bayesian correctness, independent of accuracy or perplexity.

In later sections we show that transformers achieve near-perfect calibration, while matched MLPs do not.

\section{Experimental Design}
\label{sec:experiments}

We evaluate whether small transformers can realize exact Bayesian inference by placing them in four
controlled ``Bayesian wind tunnels'' where memorization is computationally infeasible and the analytic posterior is
known in closed form. The four tasks---bijection learning, Hidden Markov Model (HMM) state tracking, Bayesian regression, and associative recall---
probe different inference structures. Bijections require discrete hypothesis elimination; HMMs require 
recursive integration of stochastic transitions and emission likelihoods.

Across both settings, the evaluation criterion is simple:  
\emph{does the model's predictive entropy $H_{\text{model}}$ match the analytic posterior entropy 
$H_{\text{Bayes}}$ at every position?}

We measure this using mean absolute entropy error (MAE),
\begin{equation}
\text{MAE}=\frac{1}{L}\sum_{t=1}^{L} 
\bigl| H_{\text{model}}(t) - H_{\text{Bayes}}(t) \bigr|,
\label{eq:mae_def}
\end{equation}
where $L$ is the number of supervised prediction positions. Because each training instance uses a fresh
bijection or a fresh HMM, memorization is infeasible; the model must perform genuine in-context inference.

\subsection{Task 1: Bijection Learning}
\label{sec:exp-bijection}

Each sequence is derived from a new random bijection 
$\pi:\{1,\dots,V\}\rightarrow\{1,\dots,V\}$ with $V=20$.  
At position $k$, the model has observed $k-1$ distinct input--output pairs and must predict 
$\pi(x_k)$. Because inputs never repeat, the Bayes-optimal posterior over $\pi(x_k)$ is uniform over the 
$V-k+1$ unseen values.

\paragraph{Bayesian ground truth.}
Let $\mathcal{O}_{k-1}$ be observed outputs. Then
\[
p(\pi(x_k)=y \mid \text{context})=
\begin{cases}
\frac{1}{V-k+1}, & y\notin\mathcal{O}_{k-1},\\
0, & y\in\mathcal{O}_{k-1},
\end{cases}
\]
with entropy $H_{\text{Bayes}}(k)=\log_2(V-k+1)$.

\paragraph{Evaluation.}
We compute MAE over a held-out set of 2{,}000 bijections.  
Because $20!\approx 2.4\times 10^{18}$ possible bijections exist and training uses only 
$10^5$ samples, no bijection is seen twice; the task enforces true hypothesis elimination.

\paragraph{Sequence format.}
Each training example is tokenized as
\[
[x_1,\,y_1,\,\mathrm{SEP},\,x_2,\,y_2,\,\mathrm{SEP},\,\dots,\,x_{19},\,\mathrm{SEP}],
\]
with teacher forcing at every $y_k$ position.

\subsection{Task 2: Hidden Markov Model State Tracking}
\label{sec:exp-hmm}

The second wind tunnel probes a qualitatively different inferential structure: recursive belief 
updating. Each sequence is derived from a fresh HMM with $S=5$ hidden states and $V=5$ observation 
symbols. Transition rows and emission rows are drawn independently from a symmetric
Dirichlet distribution with all concentration parameters equal to 1 (i.e.,
$\text{Dirichlet}(1,1,1,1,1)$),
ensuring diverse and non-degenerate dynamics. The initial state distribution $\pi_0$ is also drawn from $\text{Dirichlet}(1,1,1,1,1)$.

\paragraph{Sequence format.}
Each sequence contains:
\begin{itemize}[leftmargin=1.2em]
\item a 10-token \textbf{header} encoding flattened $T$ and $E$, and
\item $K$ observation--prediction pairs, each consisting of:
  \begin{itemize}
  \item the observed symbol $o_t$,
  \item a supervised prediction at that same position for $p(s_t\mid o_{1:t})$.
  \end{itemize}
\end{itemize}

\paragraph{Bayesian ground truth: forward algorithm.}
For each HMM and for each time $t$ we compute
\begin{equation}
\alpha_t(s)\propto E(o_t\mid s)\sum_{s'} T(s\mid s')\alpha_{t-1}(s'),
\label{eq:forward}
\end{equation}
normalized to $\sum_s\alpha_t(s)=1$.  
The exact posterior entropy is
\[
H_{\text{Bayes}}(t)=-\sum_{s=1}^{S}\alpha_t(s)\log_2\alpha_t(s).
\]

\paragraph{Evaluation lengths.}
Models are trained on sequences with $K=20$ prediction positions and evaluated on:
\begin{itemize}[leftmargin=1.2em]
\item $K=20$ (validation: within training horizon),
\item $K=30$ ($1.5\times$ training length),
\item $K=50$ ($2.5\times$ training length).
\end{itemize}
This tests whether the model has learned a position-independent recursive algorithm or has merely
memorized a finite-horizon computation.

\paragraph{Why memorization is computationally infeasible.}
Each sequence uses new $T$, $E$ matrices and new stochastic emission trajectories.
The space of possible HMMs exceeds $10^{40}$ even under coarse discretization, ensuring that learned
behavior cannot rely on recall of any particular HMM.

\subsection{Task 3: Associative Recall}
\label{sec:exp-recall}

To isolate the \emph{random-access binding} primitive, we use a standard associative recall task: the model must retrieve a target given a probe cue from a set of cue--target pairs presented in context.

\paragraph{Task format.}
Each sequence contains $N$ cue--target pairs $(c_i, t_i)$ followed by $Q$ probe cues. Cues and targets are drawn uniformly from disjoint vocabularies of size 256 each (total vocabulary 522 including special tokens). The model must predict the associated target for each probe. With $N=16$ pairs and $Q=3$ probes, sequence length is 137 tokens.

\paragraph{Experimental setup.}
\begin{itemize}[leftmargin=1.2em]
\item \textbf{Training:} 50,000 sequences, 30 epochs
\item \textbf{Models:} Transformer, Mamba, LSTM with $d_{\mathrm{model}}=128$, comparable parameter counts ($\sim$570k--660k)
\item \textbf{Metric:} Exact-match accuracy on probe predictions
\end{itemize}

\paragraph{Why this tests binding.}
Unlike bijection (where the hypothesis space shrinks predictably) or HMM (where beliefs evolve through dynamics), associative recall requires \emph{late, content-dependent retrieval}: the model cannot know which cue--target pairs matter until the probe arrives. This requires storing all pairs and retrieving by content---the binding primitive.

\paragraph{Results preview.}
At $N=16$ pairs: Transformer achieves 100\% accuracy (by epoch 12); Mamba achieves 97.8\% (epoch 30); LSTM achieves 0.5\% (random chance). The 2.5$\times$ longer training for Mamba and its imperfect accuracy suggest that selective SSM routing can approximate but not fully implement random-access binding. Full results appear in \Cref{sec:results-arch}.

\subsection{Task 4: Bayesian Regression}
\label{sec:exp-regression}

To test continuous latent variables, we use multivariate linear regression with a Gaussian prior.

\paragraph{Task format.}
Each sequence derives from a fresh linear regression with weights $w \sim \mathcal{N}(0, I_d)$ where $d=3$. Observations follow $y_i = w^\top x_i + \epsilon_i$ with $\epsilon_i \sim \mathcal{N}(0, 0.25)$ and inputs $x_i \sim \mathcal{N}(0, I_d)$. The model observes $k=6$ context pairs before predicting at a query point drawn from the same input distribution. Outputs are discretized into 41 bins uniformly spaced over $[-5, 5]$, covering $>99.9\%$ of the predictive mass given the prior and noise scales.

\paragraph{Bayesian ground truth.}
The posterior over $w$ given context is Gaussian with closed-form mean and covariance, yielding a Gaussian predictive distribution at any query point. This enables exact computation of the Bayesian predictive for comparison.

\paragraph{Why this tests continuous inference.}
Unlike bijection (discrete hypotheses) or HMM (discrete states), regression requires inference over a continuous parameter space. The Gaussian structure ensures the posterior remains tractable while testing whether transformers can calibrate uncertainty in continuous settings.

\subsection{Architectures}
\label{sec:exp-architectures}

\paragraph{Transformers.}  
We use small but realistic transformer stacks:

\begin{itemize}[leftmargin=1.2em]
\item \textbf{Bijection transformer (2.67M):}  
6 layers, 6 heads, $d_{\mathrm{model}}=192$, $d_{\mathrm{ffn}}=768$.
\item \textbf{HMM transformer (2.68M):}  
9 layers, 8 heads, $d_{\mathrm{model}}=256$, $d_{\mathrm{ffn}}=1024$.
\end{itemize}

Both use learned token embeddings, learned absolute positional embeddings, pre-norm residual blocks, 
and standard multi-head self-attention.

\paragraph{Mamba (selective state-space model).}
To test whether attention specifically is required, or whether content-based routing more generally suffices, we train Mamba models \citep{gu2024mamba}:
\begin{itemize}[leftmargin=1.2em]
\item \textbf{Bijection Mamba (3.77M):} 9 layers, $d_{\mathrm{model}}=256$, state dimension 16.
\item \textbf{HMM Mamba (3.77M):} 9 layers, $d_{\mathrm{model}}=256$, state dimension 16.
\end{itemize}
Mamba replaces attention with a selective state-space mechanism: input-dependent matrices $(\Delta, B, C)$ gate the recurrent state update. This provides content-based routing without explicit query-key matching.

\paragraph{LSTM baselines.}
To test whether recurrence alone suffices, we train LSTMs:
\begin{itemize}[leftmargin=1.2em]
\item \textbf{Bijection LSTM (4.77M):} 9 layers, hidden dimension 256.
\item \textbf{HMM LSTM (4.77M):} 9 layers, hidden dimension 256.
\end{itemize}
LSTMs have recurrent state but use fixed gating---forget, input, and output gates do not depend on content relationships across sequence positions. This tests whether recurrence without content-based routing can implement Bayesian inference.

\paragraph{Capacity-matched MLP baselines.}
To isolate the role of sequence structure entirely, we train MLPs with 18--20 layers, width 384--400, residual connections and layer normalization (parameter counts match transformers within 1\%). The MLP receives the \emph{entire context sequence} (all previous tokens concatenated) as input at each prediction position, but processes this flattened input without any attention or recurrence---testing whether feedforward computation over concatenated context suffices for Bayesian inference. We chose this design to give the MLP maximal access to context information; stronger attention-free baselines (e.g., Perceiver-style pooling, permutation-invariant aggregation) could narrow the gap but would reintroduce forms of content-based routing.

\subsection{Training Protocol}
\label{sec:exp-training}

Training is identical across architectures for each task.

\paragraph{Optimization.}
AdamW with $\beta_1=0.9$, $\beta_2=0.999$, weight decay 0.01, gradient clipping at 1.0. Batch size is 64 for all tasks.

\paragraph{Learning rates and training steps.}
\begin{itemize}[leftmargin=1.2em]
\item Bijections: constant $10^{-3}$ for 150k steps.
\item HMMs: $3\times 10^{-4}$ with 1000-step warmup and cosine decay for 100k steps.
\item Regression: $5\times 10^{-4}$ for 50k steps.
\end{itemize}

\paragraph{Data sampling.}
Every batch draws fresh bijections or fresh HMMs; sequences never repeat.

\paragraph{Teacher forcing.}
Cross-entropy loss is applied at each supervised prediction position.

\paragraph{Ablation stability.}
Layer-wise and head-wise ablations are reported as averages over three random seeds; the HMM 
length-generalization results are also evaluated across multiple seeds to ensure robustness.

\section{Results: Transformers Track the Bayesian Posterior}
\label{sec:results}

We evaluate whether transformers lie on the analytic Bayesian manifold using two behavioral tests:
(1) pointwise calibration---does $H_{\text{model}}(t)$ match $H_{\text{Bayes}}(t)$ at every position?  
(2) generalization---does the learned computation extend to unseen bijections, unseen HMMs, and longer sequences?

We present results for bijections and HMMs in parallel, followed by MLP controls and multi-seed robustness.

\subsection{Bijection Wind Tunnel: Exact Hypothesis Elimination}
\label{sec:results-bijection}

A 2.67M-parameter transformer converges to the analytic posterior with near machine precision.
\Cref{fig:transformer_vs_mlp} shows the predictive entropy
\[
H_{\text{model}}(k)=-\sum_y p_{\text{model}}(y\mid x_k,\text{context})\log_2 p_{\text{model}}(y\mid x_k,\text{context})
\]
overlaid on $H_{\text{Bayes}}(k)=\log_2(V-k+1)$.  
The curves coincide across all positions, including late steps where only 2--4 hypotheses remain.

Quantitatively, the transformer achieves
\[
\text{MAE}=3\times 10^{-3} \;\text{bits},
\]
averaged over 2{,}000 held-out bijections.
This error is smaller than single-precision numerical noise in the analytic posterior. All analytic Bayes computations use double-precision arithmetic, so reported errors are meaningful. We verify full distributional agreement via KL and TVD: across all positions, $\text{KL}(\text{model}\|\text{Bayes}) < 0.01$ nats and TVD $< 3\%$, with per-position analysis confirming agreement holds across the full entropy range (high, moderate, and low entropy positions).

\paragraph{Per-sequence evidence.}
Aggregate calibration could hide averaging artifacts.  
\Cref{fig:grid_hybrid} plots eight individual entropy trajectories.  
Each displays the characteristic staircase pattern:
entropy drops discretely whenever a new input--output pair eliminates hypotheses,  
and collapses to near zero when an input repeats and the mapping is known.  
The model performs stepwise Bayesian elimination, reproducing the curve sequence by sequence rather than merely matching it in expectation.

\paragraph{Inside-model consistency.}
Layer-wise ablations (\Cref{fig:layer0_vs_ffn}) show that removing any block increases error by 
more than an order of magnitude, confirming a deeply compositional computation.  
Head-wise ablations (\Cref{fig:head_ablation}) identify a single Layer~0 hypothesis-frame head whose removal is uniquely destructive,
consistent with the geometric analysis in \Cref{sec:mechanism}.

\subsection{HMM Wind Tunnel: Recursive Bayesian State Tracking}
\label{sec:results-hmm}

The 2.68M-parameter transformer also learns the forward algorithm for HMM inference.

\paragraph{Within training horizon (K=20).}
At $t\leq 20$, model entropy tracks the exact forward-recursion entropy with
\[
\text{MAE}=7.5\times 10^{-5} \;\text{bits}.
\]
The two curves are visually indistinguishable (\Cref{fig:hmm_entropy_lengths}). (Note: The architecture comparison in \Cref{tab:arch_comparison} uses a smaller parameter-matched model at 2.7M for fair cross-architecture comparison, yielding MAE $= 0.049$ bits; the result here uses the task-optimized 2.68M transformer.)

\paragraph{Beyond training horizon (K=30, K=50).}
To test algorithmic generalization, we roll the model out to $1.5\times$ and $2.5\times$ training length.  
The transformer stays close to the analytic posterior:
\[
\text{MAE}(K=30)=1.25\times 10^{-2},\qquad
\text{MAE}(K=50)=2.88\times 10^{-2}.
\]
Errors increase smoothly with $t$, with \emph{no} discontinuity at $t=20$ (the training boundary).  
This is strong evidence of a position-independent recursive algorithm rather than a finite-horizon memorized computation.

\paragraph{Per-position calibration.}
\Cref{fig:hmm_perpos_error} shows absolute error $|H_{\text{model}}(t)-H_{\text{Bayes}}(t)|$.  
Three patterns emerge:
\begin{enumerate}[leftmargin=1.2em]
\item early positions are slightly noisier (uncertain initial state);
\item mid-sequence positions achieve near-zero error at all lengths;
\item late positions degrade smoothly with sequence length, consistent with accumulated numerical drift.
\end{enumerate}

\paragraph{Per-sequence dynamics.}
\Cref{fig:hmm_entropy_grid} shows the model tracking sequence-specific fluctuations:  
entropy dips when emissions strongly identify states and rises when observations are ambiguous.  
The transformer captures these dynamics exactly.

\paragraph{Semantic invariance under hidden-state relabeling.}
Hidden-state indices are purely symbolic: permuting the labels corresponds to the same latent process.
We sample a random permutation $\sigma$ of $\{1,\dots,S\}$ and apply it to the HMM parameters by permuting
rows and columns of $T$ (i.e., $T'_{\sigma(i),\sigma(j)} = T_{i,j}$) and permuting rows of $E$
(i.e., $E'_{\sigma(i),o} = E_{i,o}$). We then recompute the analytic posterior under $(T',E')$ and
evaluate the model on sequences generated from the permuted HMM. If the model implements Bayesian filtering
rather than associating meaning with specific state IDs, its entropy calibration should be unchanged up to
numerical noise. \Cref{fig:hmm_perm_invariance} shows MAE before vs.\ after permutation lies on the diagonal,
with $\Delta$MAE concentrated near zero.

\subsection{Length Generalization Requires Late-Layer Attention}
\label{sec:results-hmm-no-late-attn}

To identify which components support stable rollout, we train a variant transformer in which 
attention is disabled in the top two layers but FFNs and residuals remain intact.

The no-late-attention model fits the training horizon reasonably well 
($1.57\times 10^{-3}$ bits),  
but breaks down under rollout:
\[
\text{MAE}(K=30)=5.55\times 10^{-1},\qquad
\text{MAE}(K=50)=1.79.
\]
The degradation factor grows from $21\times$ (at $K=20$) to $62\times$ (at $K=50$), 
demonstrating that late-layer attention is not required for fitting 
$K=20$ but \emph{is essential} for stable long-horizon Bayesian updates 
(\Cref{fig:hmm_no_late_attn_scaling}).

\subsection{Associative Recall: Content-Based Retrieval}
\label{sec:results-recall}

The associative recall task tests whether architectures can retrieve stored information by content---the \emph{binding} primitive. Unlike bijection or HMM tasks where the relevant context is determined by position or dynamics, here the model must identify which cue--target pair to retrieve based on a probe that arrives only at test time.

The transformer achieves 100\% accuracy on held-out sequences, demonstrating perfect content-based retrieval. This requires the attention mechanism: the probe cue must match against stored cues to route information from the corresponding target. MLPs did not succeed on this task under our training protocol (accuracy at chance), as they process each position independently without cross-token interaction.

\subsection{Continuous Bayesian Regression}
\label{sec:results-regression}

To test continuous latent variables, we evaluate on multivariate linear regression with a Gaussian prior. The transformer (857k params, 4 layers) predicts discretized outputs over 41 bins, achieving KL $= 0.034$ nats to the closed-form Bayesian predictive; the capacity-matched MLP achieves KL $= 0.22$---6$\times$ worse. The gap is smaller than for discrete tasks, but the transformer still substantially outperforms on uncertainty calibration, demonstrating that the Bayesian inference capability extends to continuous parameter spaces.

\subsection{Architectural Comparison: Which Primitives Does Each Architecture Realize?}
\label{sec:results-arch}

To understand which architectural ingredients enable Bayesian inference, we decompose it into three \emph{inference primitives} and test which architectures realize each:
\begin{itemize}[leftmargin=1.2em]
\item \textbf{Belief accumulation}: integrating evidence into a running posterior (tested by bijection elimination)
\item \textbf{Belief transport}: propagating beliefs through stochastic dynamics (tested by HMM filtering)
\item \textbf{Random-access binding}: retrieving stored hypotheses by content (tested by associative recall)
\end{itemize}

\Cref{tab:arch_comparison} summarizes results across all three tasks.

\begin{table}[htbp]
\centering
\caption{\textbf{Architecture comparison across inference primitives.} Bijection/HMM: entropy MAE in bits (lower is better; mean $\pm$ std over 5 seeds for Mamba/LSTM). Recall: accuracy (higher is better). Each task tests a different primitive. Transformers realize all three; Mamba realizes accumulation and transport; LSTMs realize only accumulation (of static sufficient statistics); MLPs realize none.}
\label{tab:arch_comparison}
\begin{tabular}{lcccc}
\toprule
Architecture & Bijection & HMM & Assoc.\ Recall & Primitives \\
 & (accumulation) & (transport) & (binding) & realized \\
\midrule
Transformer & 0.007 & 0.049 & \textbf{100\%} & All 3 \\
Mamba & 0.010$\pm$.001 & \textbf{0.024$\pm$.009} & 97.8\% (slow) & 2 of 3 \\
LSTM & 0.009$\pm$.000 & 0.411$\pm$.003 (fail) & 0.5\% (fail) & 1 of 3 \\
MLP & 1.85 & 0.40 & --- & 0 of 3 \\
\bottomrule
\end{tabular}
\end{table}

\paragraph{Transformers realize all three primitives.}
Transformers achieve near-exact Bayesian posteriors on bijection and HMM, and perfect accuracy on associative recall with 64 cue--target pairs. Attention provides all three capabilities: accumulation via the residual stream, transport via content-based routing, and binding via query-key matching.

\paragraph{Mamba realizes accumulation and transport, but struggles with binding.}
Mamba \emph{outperforms} the transformer on HMM tracking (0.024 vs 0.049 bits MAE), demonstrating that its selective state-space mechanism excels at belief transport. However, on associative recall, Mamba reaches only 97.8\% accuracy and requires 2.5$\times$ more training epochs than the transformer. This matches the finding of \citet{jelassi2024repeat} that SSMs struggle with retrieval tasks---Mamba's selection mechanism implements content-based routing on \emph{transition dynamics}, which enables transport but not direct random-access retrieval.

\paragraph{LSTMs realize only accumulation of static sufficient statistics.}
The LSTM achieves 0.009 bits on bijection---comparable to Transformer and Mamba---but fails on both HMM (0.416 bits) and associative recall (0.5\%, random chance). The key distinction: bijection elimination admits a \emph{static} sufficient statistic (the set of observed outputs, representable in fixed dimension) that LSTM can track with its recurrent state. But HMM requires a sufficient statistic that \emph{evolves under dynamics}---the belief vector must be transported through the transition matrix at each step---and associative recall requires the statistic to be \emph{indexed by content}. Under our training protocol, LSTM's fixed gating accumulated evidence but did not learn to implement the content-dependent operations required for transport or binding.

\paragraph{MLPs realize no primitives.}
Without sequence structure, MLPs collapse to marginal predictions and fail uniformly.

\paragraph{The primitives taxonomy predicts performance.}
This decomposition explains the full pattern of results: each architecture succeeds precisely on tasks requiring only the primitives it can implement. The taxonomy also resolves apparent contradictions in the literature---Mamba beats transformers on tasks dominated by transport (HMM) but loses on tasks requiring binding (recall).

\begin{figure}[htbp]
  \centering
  \includegraphics[width=\textwidth]{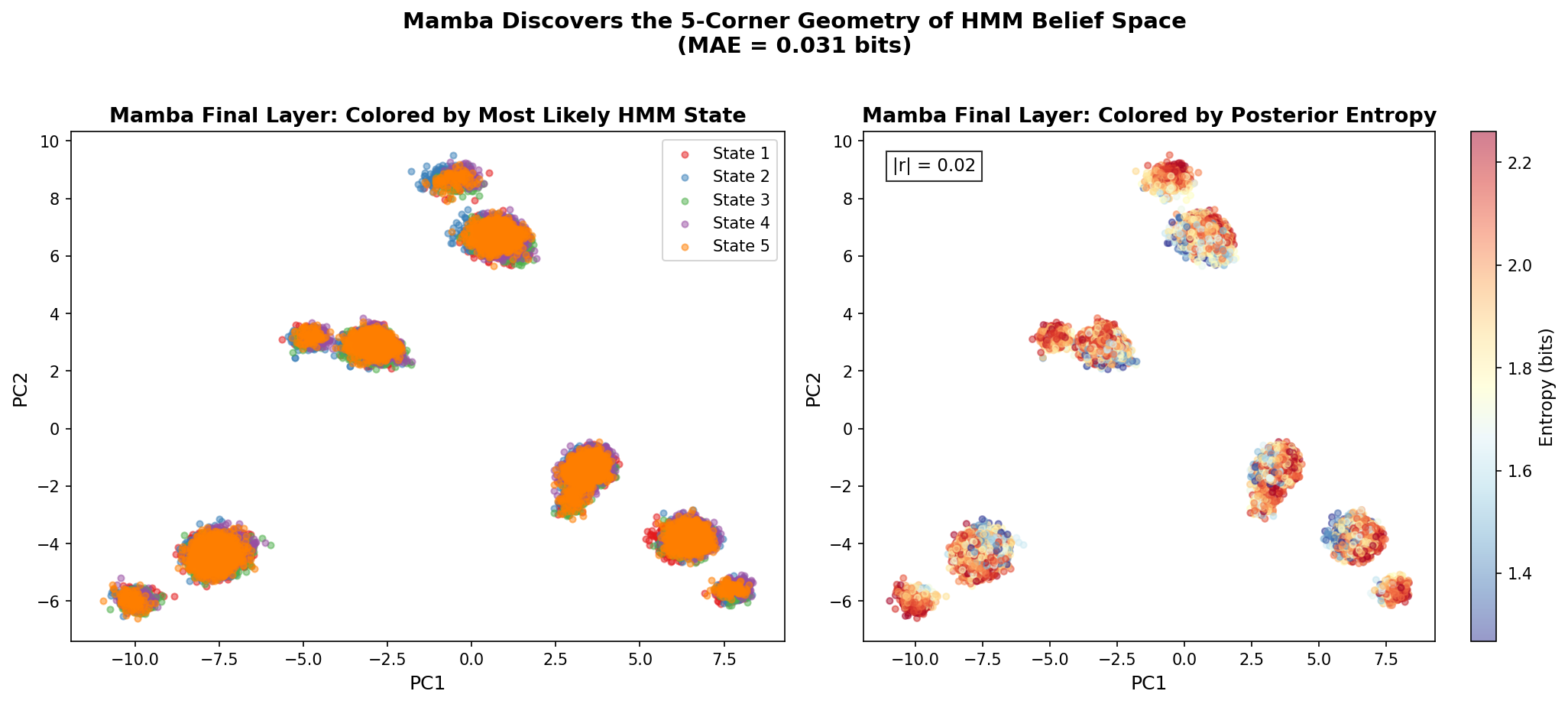}
  \caption{\textbf{Mamba discovers the 5-corner geometry of HMM belief space.}
  Final-layer representations from Mamba on the HMM task (5 hidden states).
  \emph{Left:} Points colored by most likely hidden state reveal five distinct clusters---one per state.
  \emph{Right:} The same points colored by posterior entropy show confidence variation within each cluster (red = low entropy/high confidence, blue = high entropy/uncertainty).
  Mamba has learned that the belief simplex has five corners and organized its representations accordingly. This geometric structure explains its strong performance (MAE = 0.024 bits, 5-seed average).}
  \label{fig:mamba_clusters}
\end{figure}

    \subsection{Multi-Seed Consistency}
\label{sec:results-seed}

To ensure that Bayesian tracking is not an artifact of initialization or optimization noise,
we repeated transformer HMM experiments across \textbf{five independent random seeds}.
Per-position error curves for all seeds (\Cref{fig:overlay_length_generalization_multiseed})
nearly overlap at $K=20$, $K=30$, and $K=50$.

The seed-to-seed variability is negligible compared to the gap between architectures,
confirming that the learned Bayesian algorithm is robust to initialization and
training noise. The Mamba and LSTM HMM results in \Cref{tab:arch_comparison} are 5-seed averages with standard deviations of 0.009 and 0.003 bits respectively---confirming that the large performance gaps between architectures (e.g., LSTM's 0.411 vs Mamba's 0.024 bits on HMM) reflect genuine architectural differences rather than seed-to-seed variation.

\begin{figure}[htbp]
  \centering
  \includegraphics[width=0.8\textwidth]{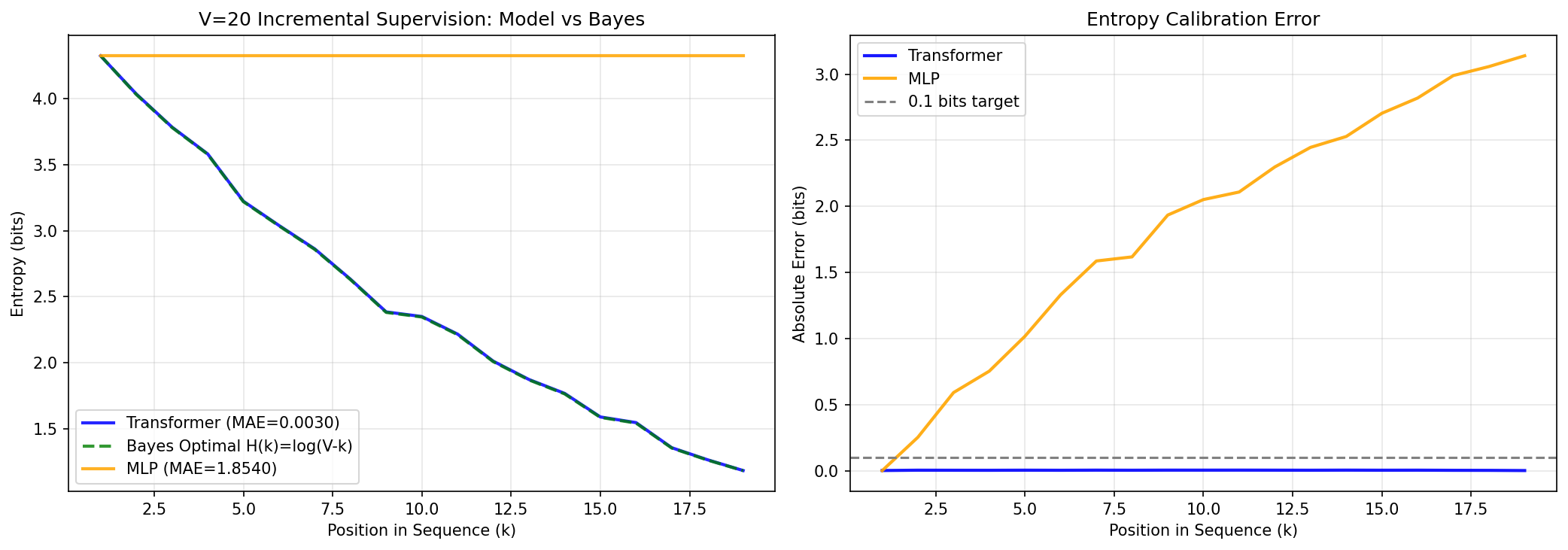}
  \caption{\textbf{Bijection wind tunnel: transformer matches the Bayesian posterior; MLP does not.}
  Entropy trajectories at 150k training steps. The transformer lies essentially on top of the analytic Bayes curve across positions, while the capacity-matched MLP barely reduces uncertainty and fails to implement hypothesis elimination. This is the comparison summarized quantitatively in \Cref{tab:hmm_transformer_vs_mlp} and discussed in \Cref{sec:results-bijection}.}
  \label{fig:transformer_vs_mlp}
\end{figure}

\begin{figure}[htbp]
  \centering
  \includegraphics[width=0.9\textwidth]{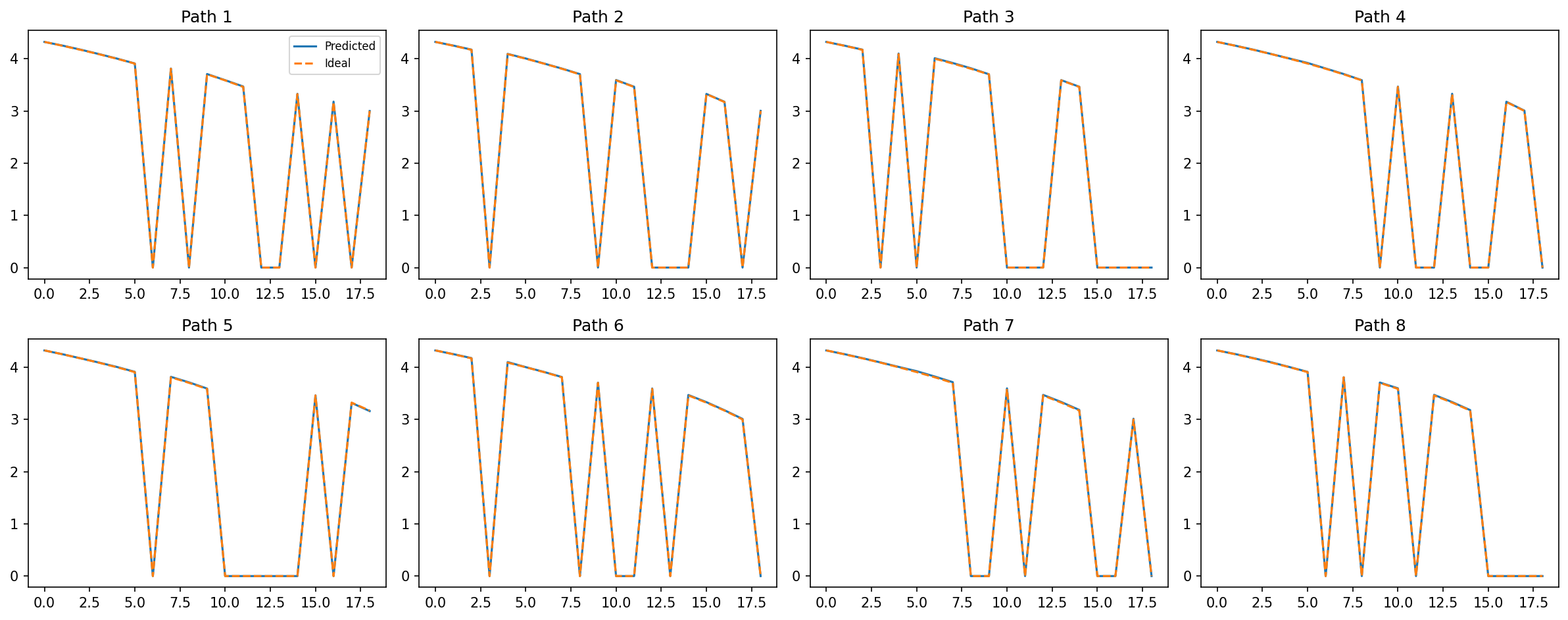}
  \caption{\textbf{Bijection wind tunnel: per-sequence entropy dynamics.}
  Eight randomly chosen bijections from the test set. Each panel shows transformer entropy (solid) and analytic Bayes entropy (dashed) as a function of position.
  The sawtooth pattern---discrete drops when mappings are revealed and collapses to (near) zero when previously seen inputs reappear---confirms that the transformer is performing stepwise hypothesis elimination, not merely matching the Bayes curve in aggregate.}
  \label{fig:grid_hybrid}
\end{figure}

\begin{figure}[htbp]
  \centering
  \includegraphics[width=0.7\textwidth]{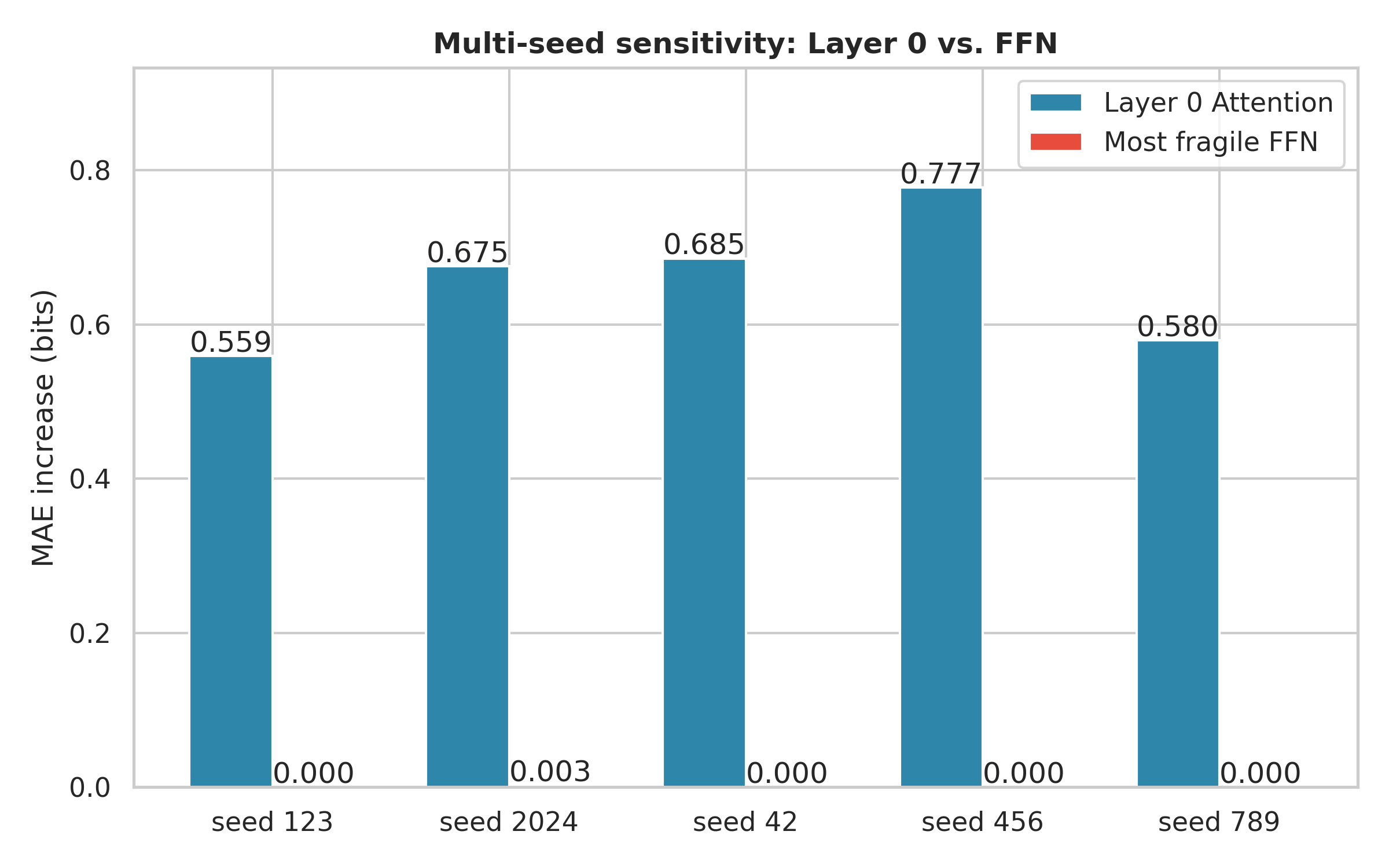}
  \caption{\textbf{Bijection wind tunnel: layer-wise ablation.}
  Mean absolute entropy error (bits) when ablating each layer (attention+FFN) in turn, averaged over seeds.
  Removing any single layer increases calibration error by more than an order of magnitude, showing that the Bayesian computation is genuinely hierarchical and compositional rather than shallow or redundant.}
  \label{fig:layer0_vs_ffn}
\end{figure}

\begin{figure}[htbp]
  \centering
  \includegraphics[width=0.5\linewidth]{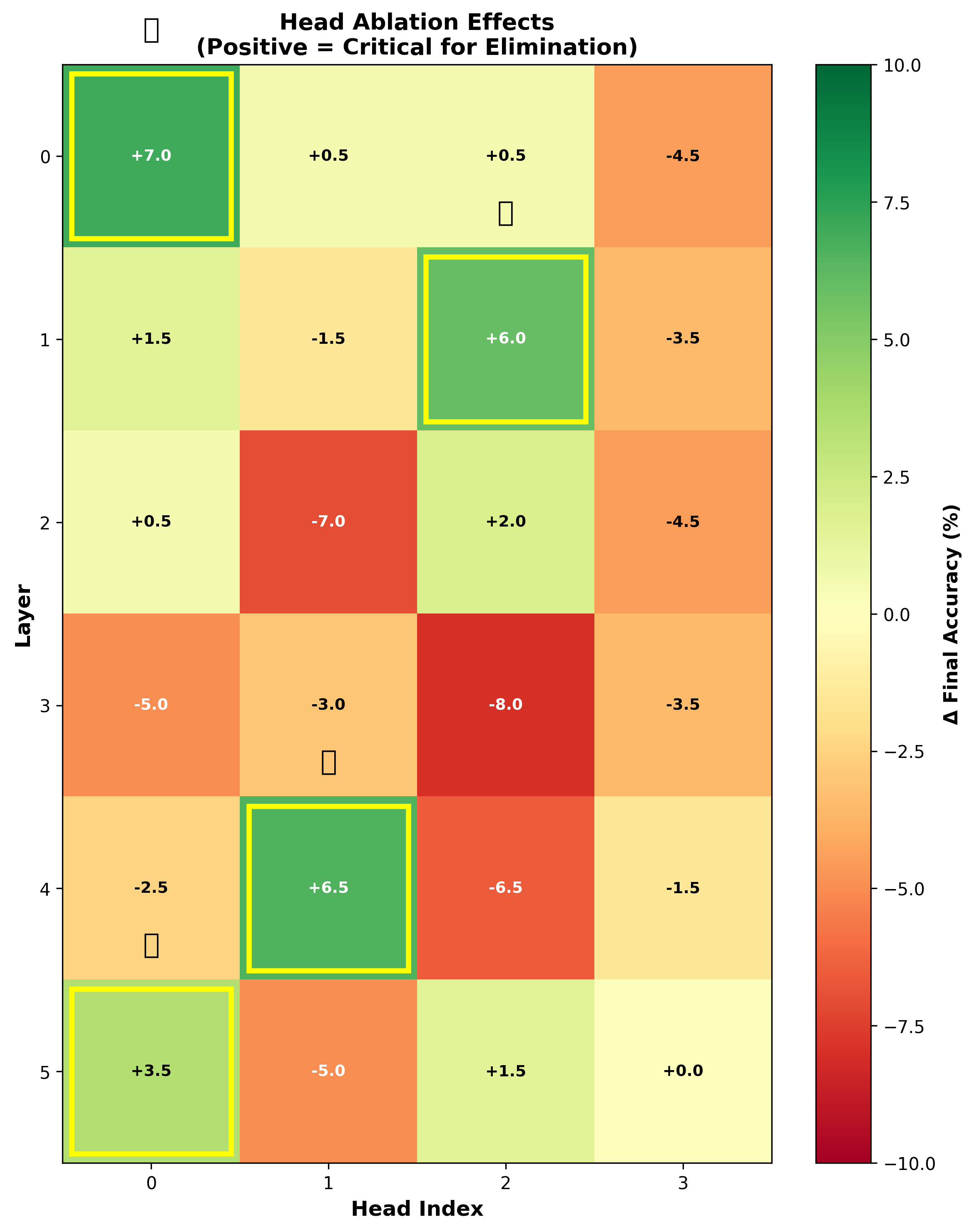}
  \caption{\textbf{Head-wise ablation.}
  Change in mean absolute entropy error when ablating individual attention heads.
  A single Layer-0 ``hypothesis-frame head'' plays a uniquely important role, while many later heads are partially redundant.
  This supports the three-stage picture in \Cref{sec:discussion}: foundational binding, progressive elimination, and value-manifold refinement.}
  \label{fig:head_ablation}
\end{figure}

\begin{figure}[htbp]
  \centering
  \includegraphics[width=0.85\textwidth]{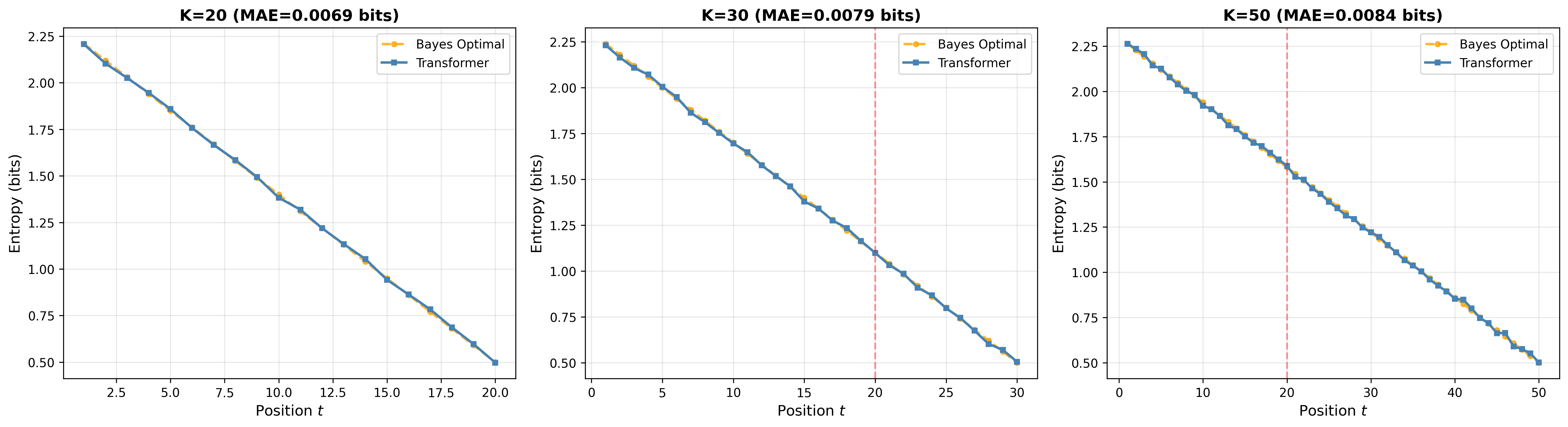}
  \caption{\textbf{HMM wind tunnel: calibration across sequence lengths.}
  Transformer predictive entropy $H_{\text{model}}(t)$ (solid) versus analytic $H_{\text{Bayes}}(t)$ (dashed) at the training length $K=20$ and at $K=30$ and $K=50$.
  At $K=20$ the trajectories overlap almost perfectly; for longer sequences the error grows smoothly with position and shows no kink at the training boundary, indicating a position-independent recursive algorithm rather than finite-horizon memorization.}
  \label{fig:hmm_entropy_lengths}
\end{figure}

\begin{figure}[htbp]
  \centering
  \includegraphics[width=0.75\textwidth]{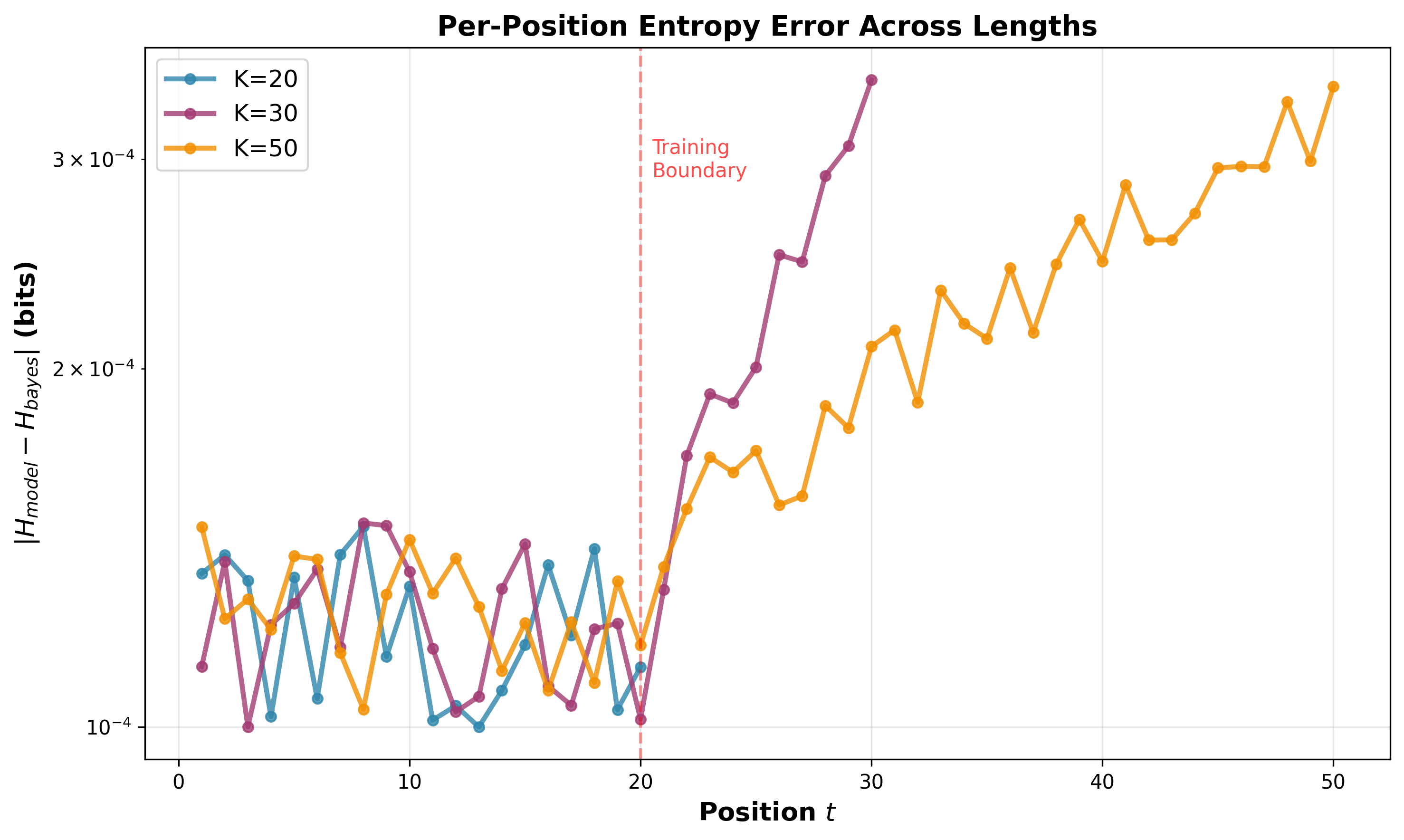}
  \caption{\textbf{HMM wind tunnel: per-position calibration.}
  Absolute entropy error $|H_{\text{model}}(t) - H_{\text{Bayes}}(t)|$ as a function of position for $K=20$, $K=30$, and $K=50$.
  Errors are tiny at the training length and increase gradually with $t$ for extended lengths, again with no discontinuity at $t=20$.}
  \label{fig:hmm_perpos_error}
\end{figure}

\begin{figure}[htbp]
  \centering
  \includegraphics[width=0.9\textwidth]{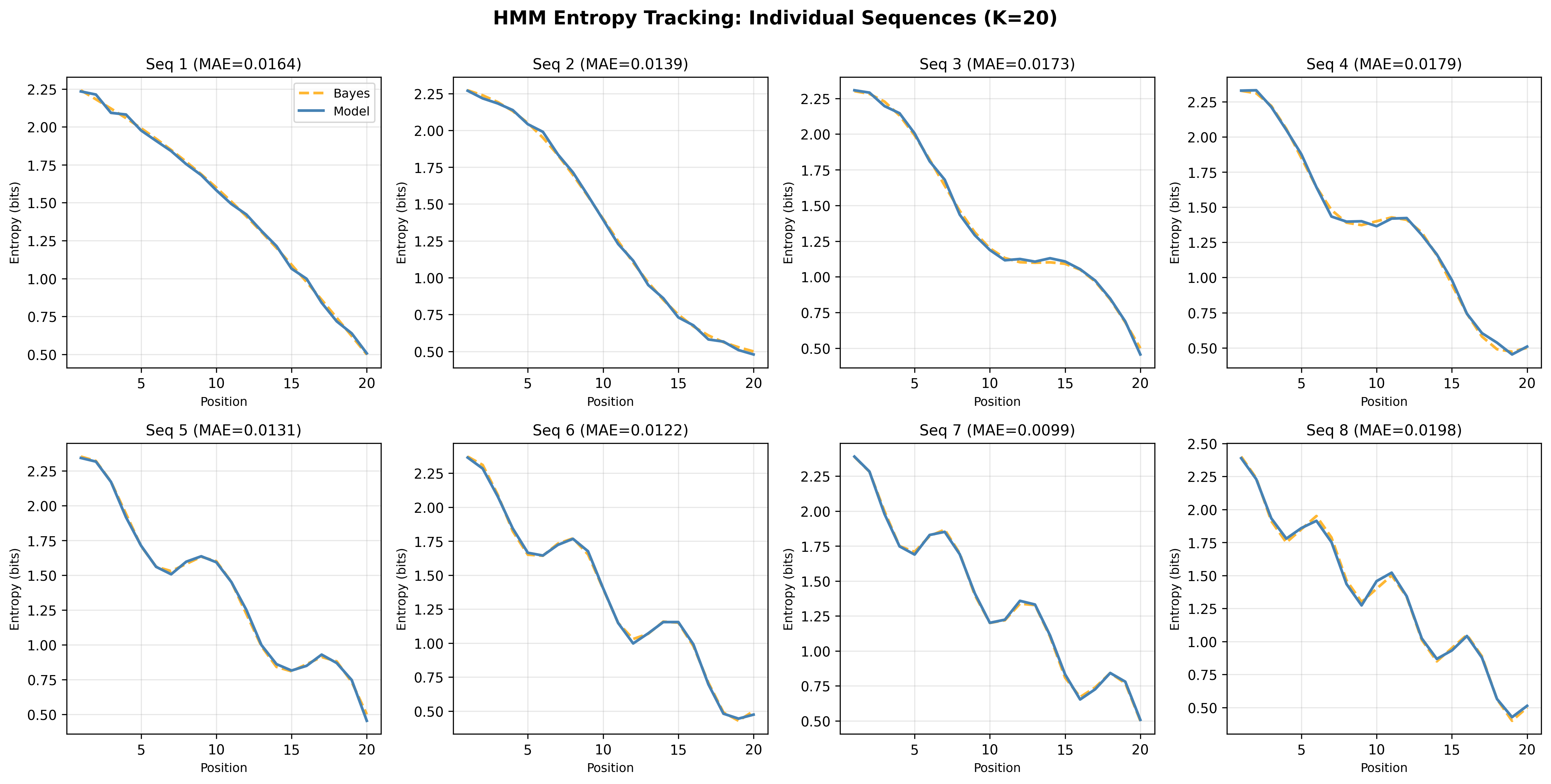}
  \caption{\textbf{HMM wind tunnel: per-sequence entropy dynamics.}
  Entropy trajectories $H_{\text{model}}(t)$ and $H_{\text{Bayes}}(t)$ for eight randomly chosen $K=20$ test HMMs.
  The transformer tracks sequence-specific rises and drops in uncertainty, reflecting the stochastic interplay of transitions and emissions.}
  \label{fig:hmm_entropy_grid}
\end{figure}

\begin{figure}[htbp]
  \centering
  \includegraphics[width=0.6\textwidth]{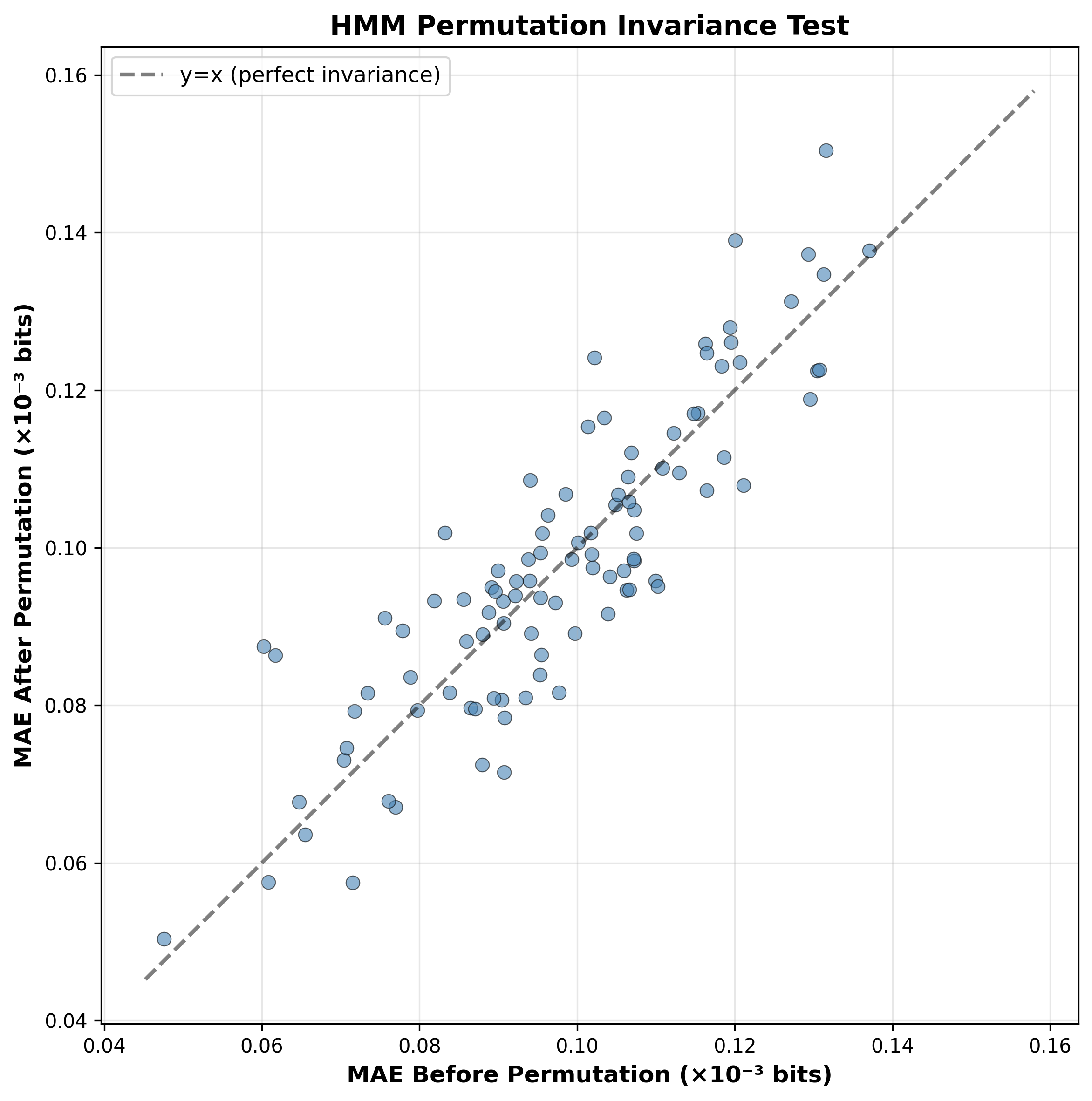}
  \caption{\textbf{Semantic invariance under hidden-state relabeling.}
  Mean absolute entropy error before vs.\ after randomly permuting hidden-state labels in the HMMs.
  Points lie on the diagonal and the distribution of $\Delta$MAE is tightly concentrated near zero, confirming that the transformer's computation is invariant to arbitrary relabelings of the hidden state space.}
  \label{fig:hmm_perm_invariance}
\end{figure}

\begin{figure}[htbp]
  \centering
  \includegraphics[width=0.7\textwidth]{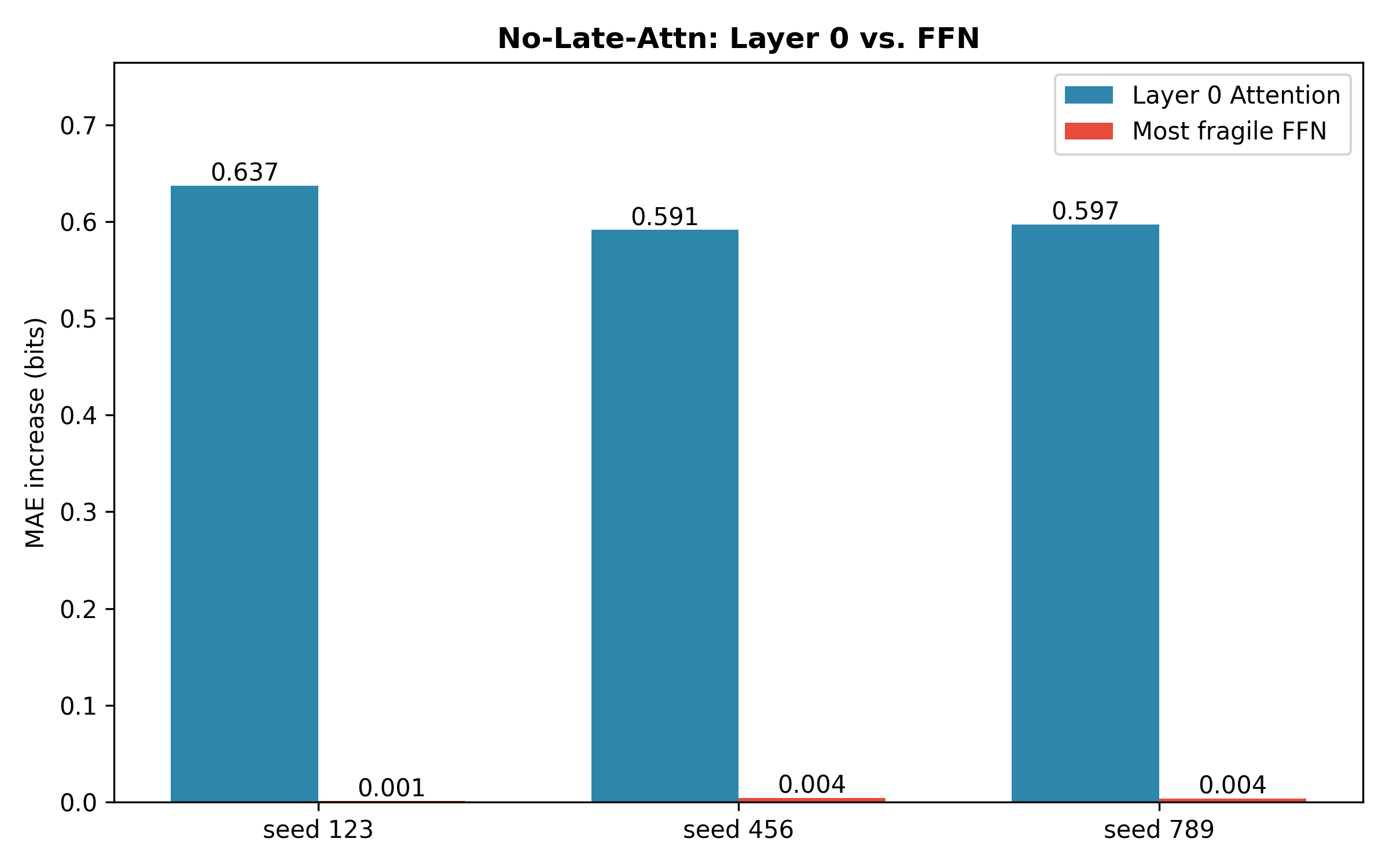}
  \caption{\textbf{Late-layer attention and length generalization.}
  Mean absolute entropy error as a function of sequence length for the full transformer and a variant with attention disabled in the top two layers.
  The no-late-attention model is only modestly worse at the training length but its error explodes on longer sequences, with the degradation factor growing from $\sim 21\times$ at $K=20$ to over $60\times$ at $K=50$.
  Late attention is therefore crucial for stable rollout beyond the training horizon.}
  \label{fig:hmm_no_late_attn_scaling}
\end{figure}

\begin{table}[htbp]
  \centering
  \caption{\textbf{HMM wind tunnel: transformer vs MLP calibration across lengths.}
  Mean absolute entropy error (bits) between model entropy and analytic Bayes entropy.
  The transformer achieves near-perfect calibration and degrades gracefully with length; the capacity-matched MLP fails catastrophically, with errors $\sim 0.4$ bits at all positions and lengths.}
  \label{tab:hmm_transformer_vs_mlp}
  \begin{tabular}{lcc}
    \toprule
    \textbf{Model} & \textbf{$K=20$ (training)} & \textbf{$K=50$ (2.5$\times$ length)} \\
    \midrule
    Transformer (2.68M) & $7.5\times 10^{-5}$ & $2.88\times 10^{-2}$ \\
    MLP (2.70M)         & $4.09\times 10^{-1}$ & $4.02\times 10^{-1}$ \\
    \midrule
    Degradation factor  & $5{,}467\times$ & $14\times$ \\
    \bottomrule
  \end{tabular}
\end{table}

\begin{figure}[htbp]
  \centering
  \includegraphics[width=\textwidth]{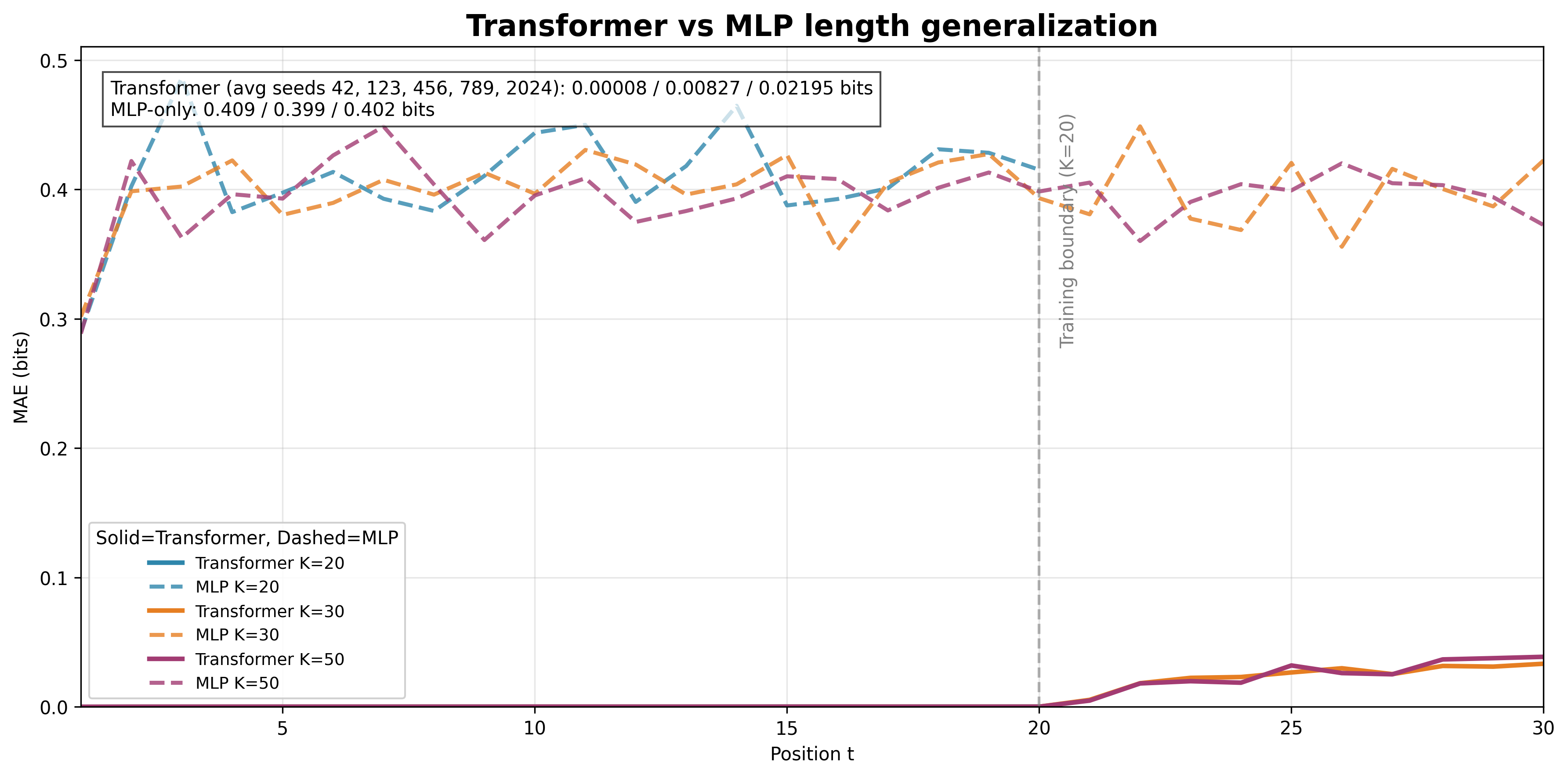}
  \caption{\textbf{HMM wind tunnel: transformer vs MLP length generalization.}
  Per-position mean absolute entropy error for the transformer (solid) and capacity-matched MLP (dashed) at $K=20$ and $K=50$.
  The vertical gray line marks the training boundary at position $t=20$.
  The transformer shows near-zero error at the training length and smooth degradation beyond it; the MLP maintains flat $\sim 0.4$-bit error across positions, indicating failure to learn recursive Bayesian updates.}
  \label{fig:hmm_length_gen_transformer_vs_mlp}
\end{figure}

\begin{figure}[htbp]
  \centering
  \includegraphics[width=\textwidth]{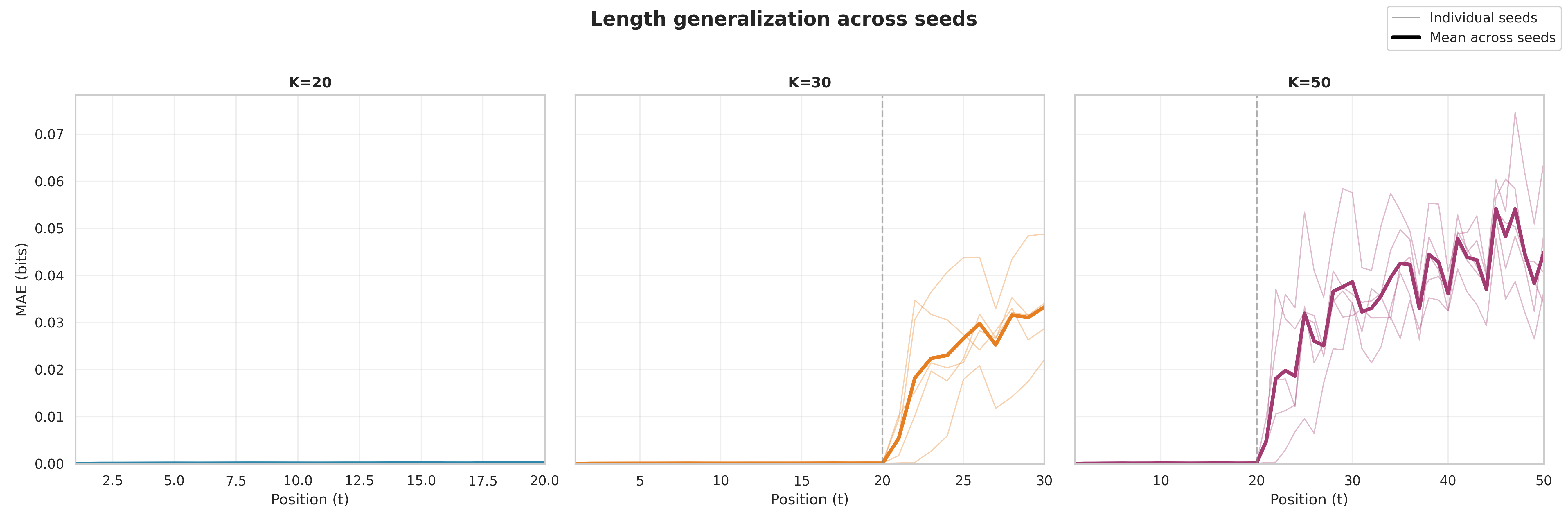}
  \caption{\textbf{Multi-seed robustness of HMM length generalization.}
  Overlay of per-position transformer MAE curves across five random seeds for $K=20$, $K=30$, and $K=50$.
  Seed-to-seed variability is negligible relative to the transformer--MLP gap, showing that the learned Bayesian algorithm is robust to initialization and optimization noise.}
  \label{fig:overlay_length_generalization_multiseed}
\end{figure}

\begin{figure}[htbp]
  \centering
  \includegraphics[width=\textwidth]{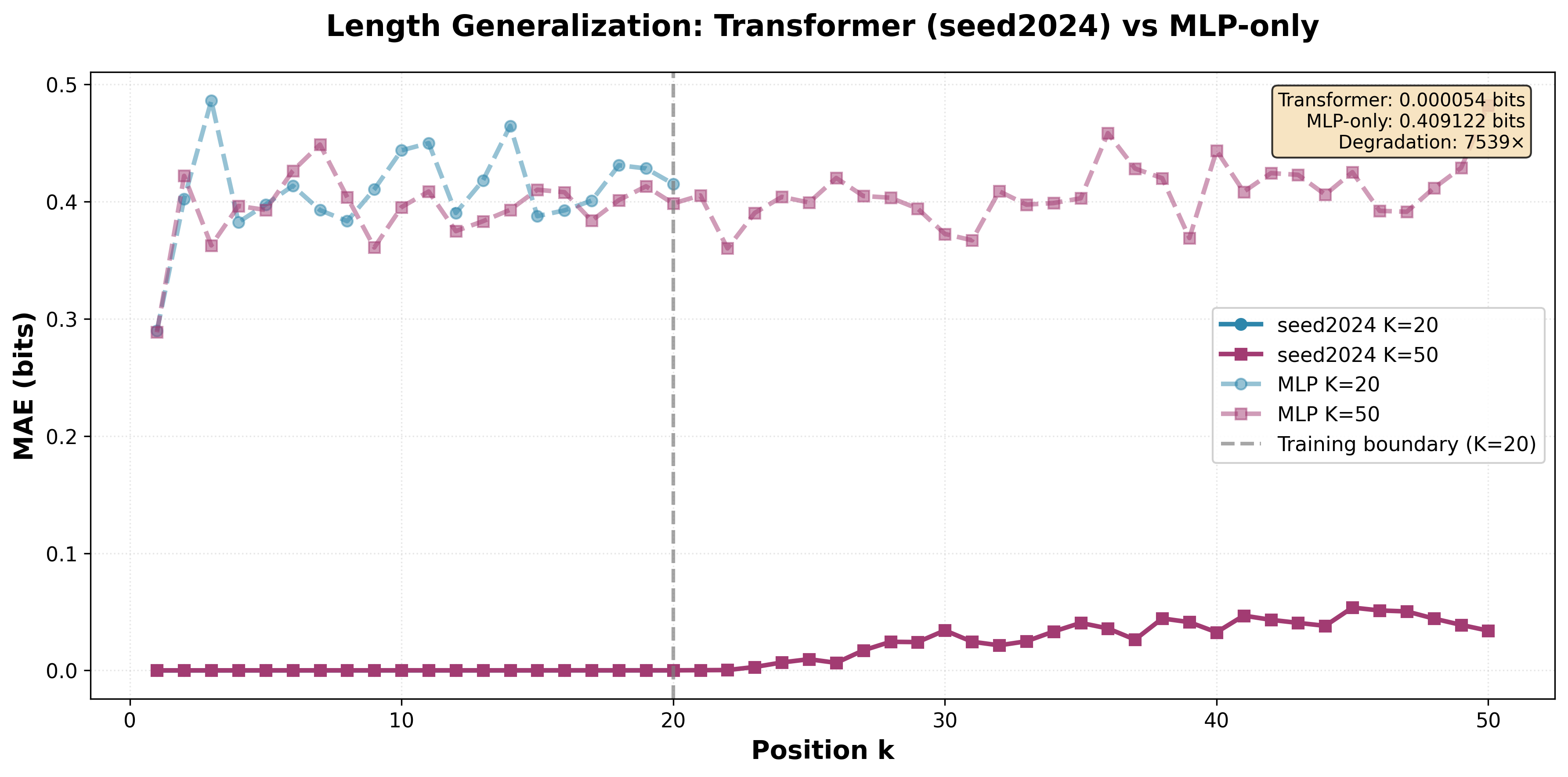}
  \caption{\textbf{Representative single-seed trajectory.}
  Per-position MAE for one representative seed (2024) closely matches the multi-seed average in \Cref{fig:overlay_length_generalization_multiseed}, further confirming that the length generalization pattern is not an artifact of a particular initialization.}
  \label{fig:overlay_length_generalization_seed2024}
\end{figure}

\begin{figure}[t]
    \centering
    \includegraphics[width=0.7\textwidth]{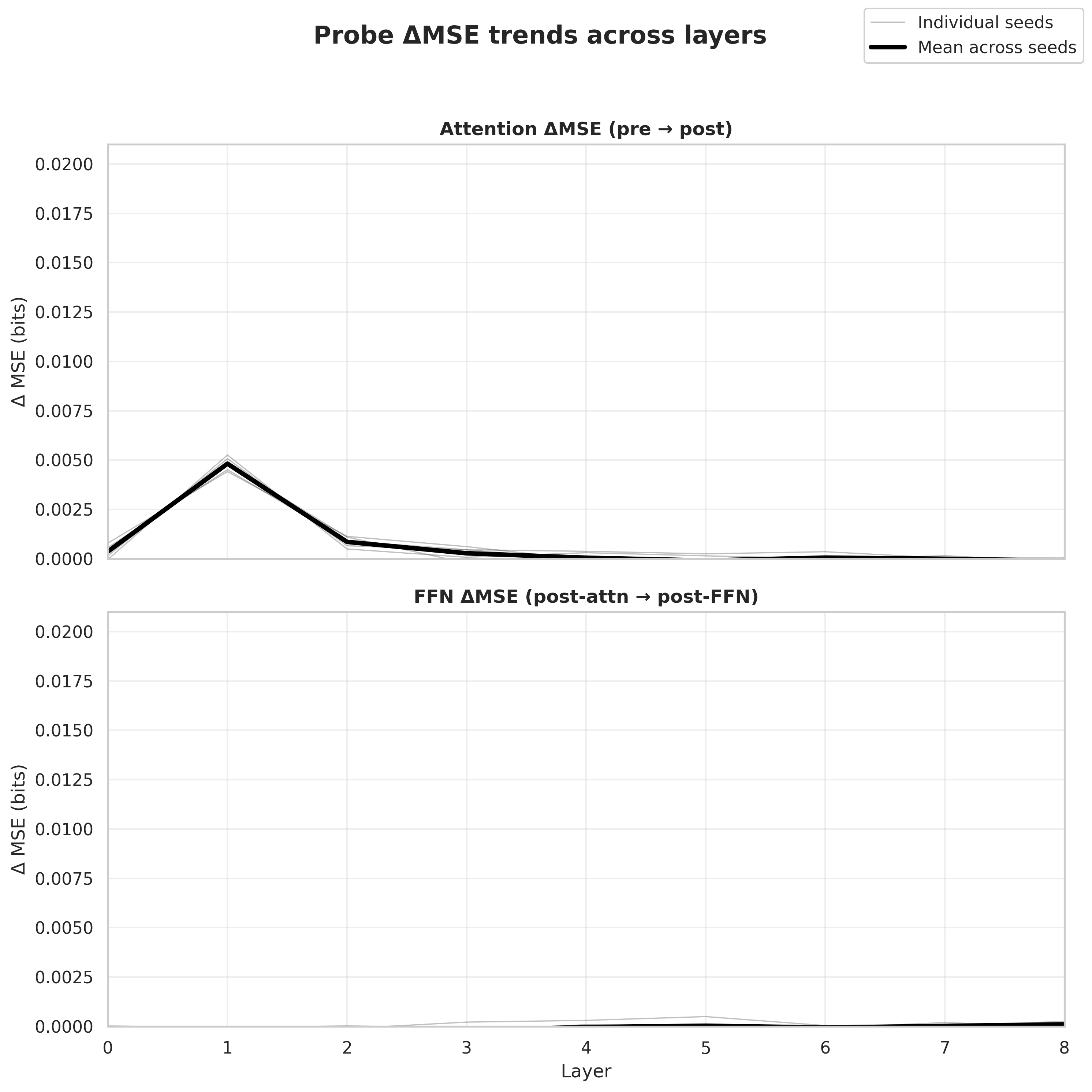}
    \caption{\textbf{Block-wise probe deltas for entropy prediction.}
    For each transformer block we train a linear probe on the \emph{pre}-sublayer residual stream to
    predict the analytic posterior entropy, then evaluate the same probe on the \emph{post}-sublayer
    residual. The plotted quantity is the change in mean-squared error (MSE) when moving from pre- to
    post-sublayer, i.e., $\Delta \mathrm{MSE} = \mathrm{MSE}(\text{probe on post-residual}) - \mathrm{MSE}(\text{probe on pre-residual})$, so negative values mean the block improves an entropy-linear probe. Positive values indicate that the
    block reduces probe error. FFN layers account for the largest reductions in MSE, showing that they
    implement most of the numerical Bayesian update, while attention primarily provides routing rather
    than performing the heavy probabilistic computation.}
    \label{fig:blockwise_probe_deltas}
\end{figure}

\begin{figure}[htbp]
  \centering
  \includegraphics[width=0.75\textwidth]{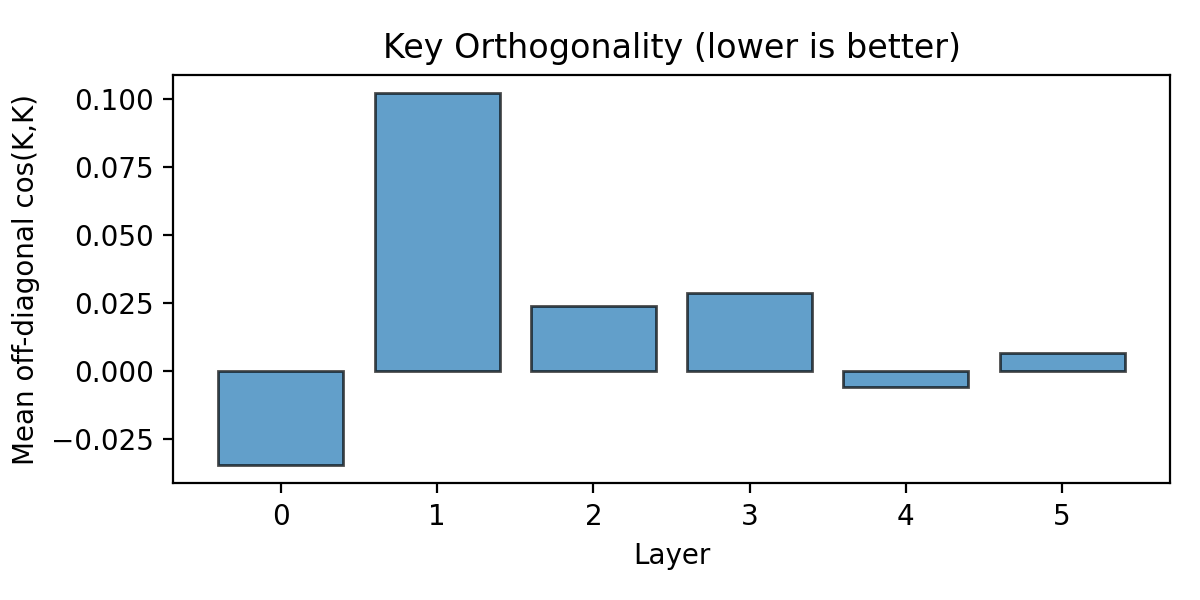}
  \caption{\textbf{Key orthogonality in Layer~0.}
  Cosine similarity matrix of key vectors for all input tokens in the bijection model at 150k steps.
  Off-diagonal entries cluster near zero, showing that distinct inputs occupy nearly orthogonal directions and form an explicit hypothesis basis.}
  \label{fig:key_orthogonality}
\end{figure}

\begin{figure}[t]
    \centering
    \begin{subfigure}[b]{0.45\textwidth}
        \centering
        \includegraphics[width=\textwidth]{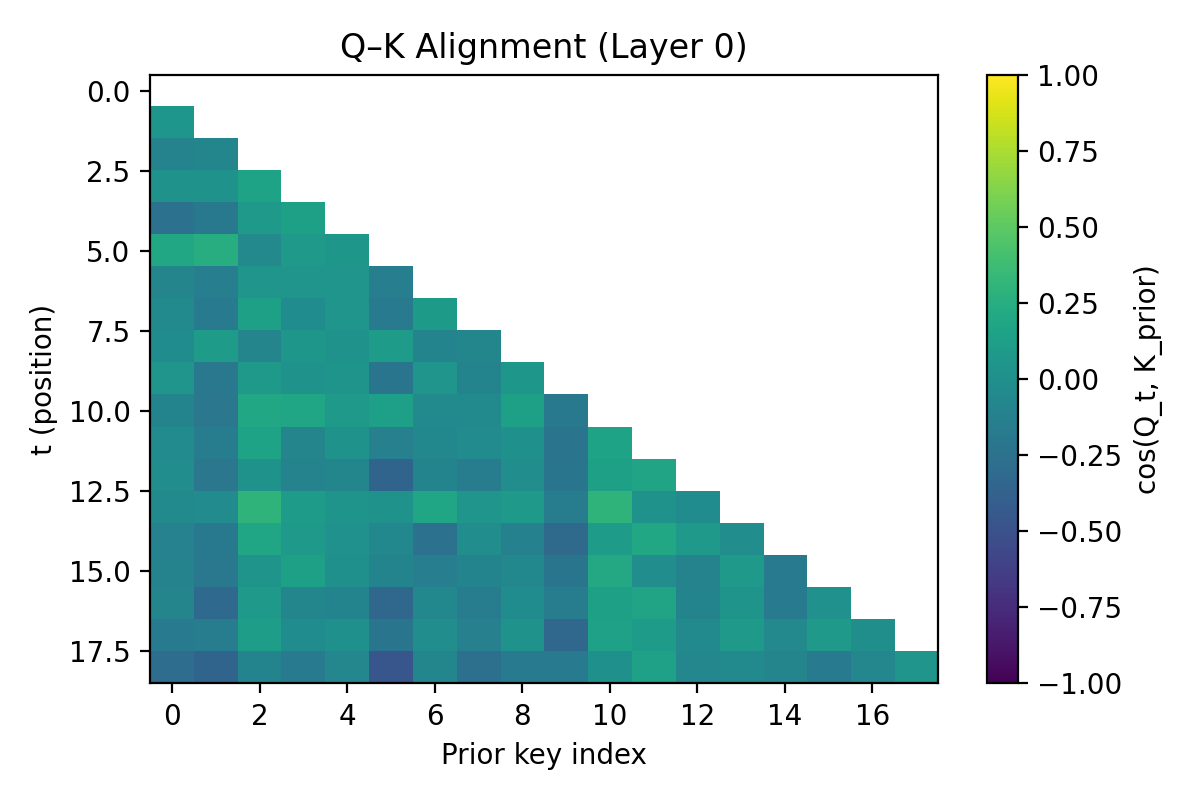}
        \caption{Layer 0}
        \label{fig:qk_layer0}
    \end{subfigure}
    \hfill
    \begin{subfigure}[b]{0.45\textwidth}
        \centering
        \includegraphics[width=\textwidth]{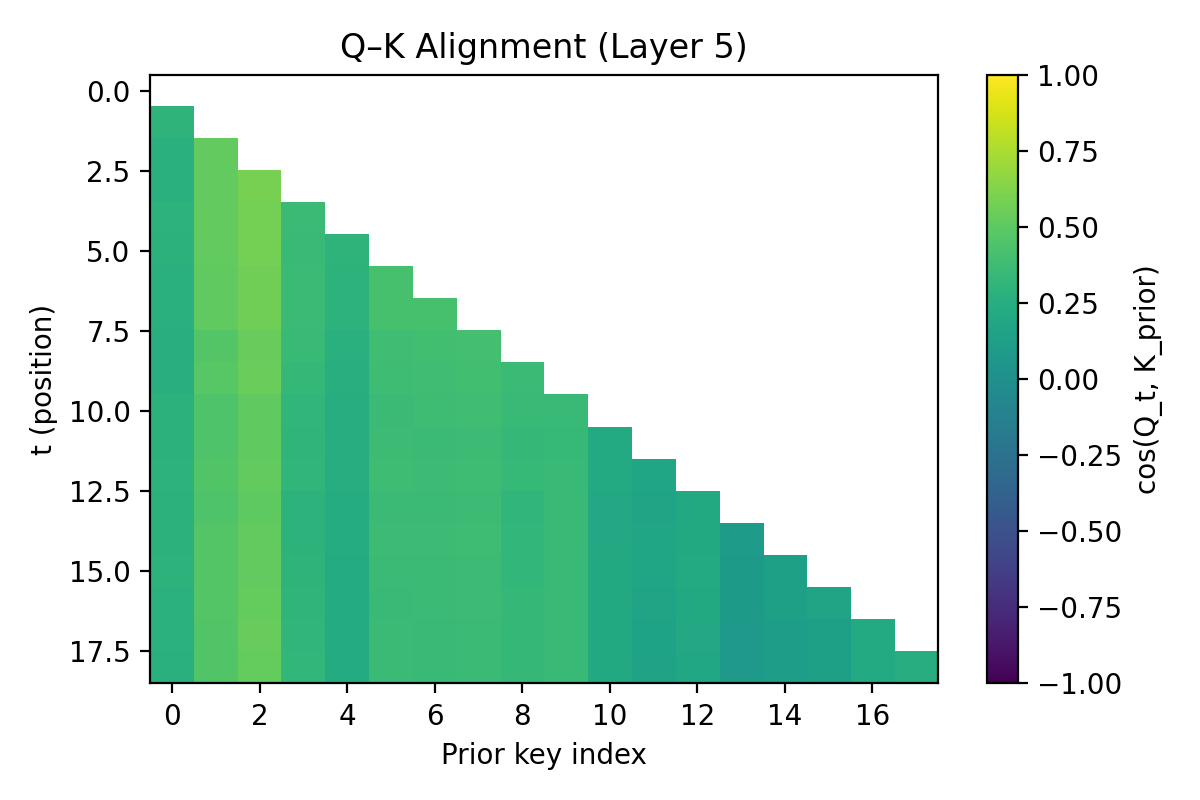}
        \caption{Layer 5}
        \label{fig:qk_layer5}
    \end{subfigure}
    \caption{\textbf{Progressive query--key alignment across depth.}
    Cosine similarity between queries and keys at an early layer (left) and a deep layer (right) of the
    bijection transformer. For each \emph{sequence position} $t$ on the horizontal axis, we plot the cosine
similarities $\cos(q_t, k_j)$ between the query at position $t$ and all key vectors $k_j$ along the
vertical axis. Here $t$ indexes the query-token positions in the serialized input sequence (i.e., the positions where the model must predict),
not the token identity; separator/header tokens are included only insofar as they occupy sequence positions.
 In Layer~0, attention is diffuse over many keys; by Layer~5 it concentrates sharply on
    the remaining feasible hypothesis keys, making sequential elimination visible as geometric focusing in
    Q--K space.}
    \label{fig:qk_alignment}
\end{figure}

\begin{figure}[htbp]
  \centering
  \begin{subfigure}{0.48\textwidth}
    \centering
    \includegraphics[width=\textwidth]{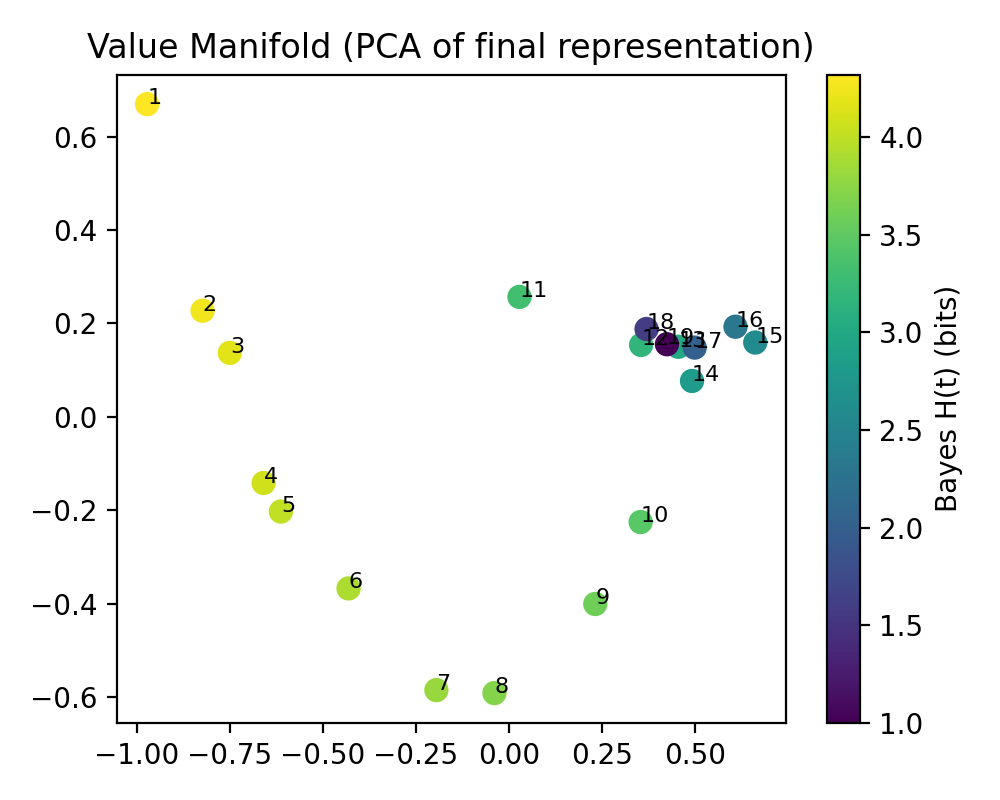}
    \caption{100k steps}
  \end{subfigure}
  \hfill
  \begin{subfigure}{0.48\textwidth}
    \centering
    \includegraphics[width=\textwidth]{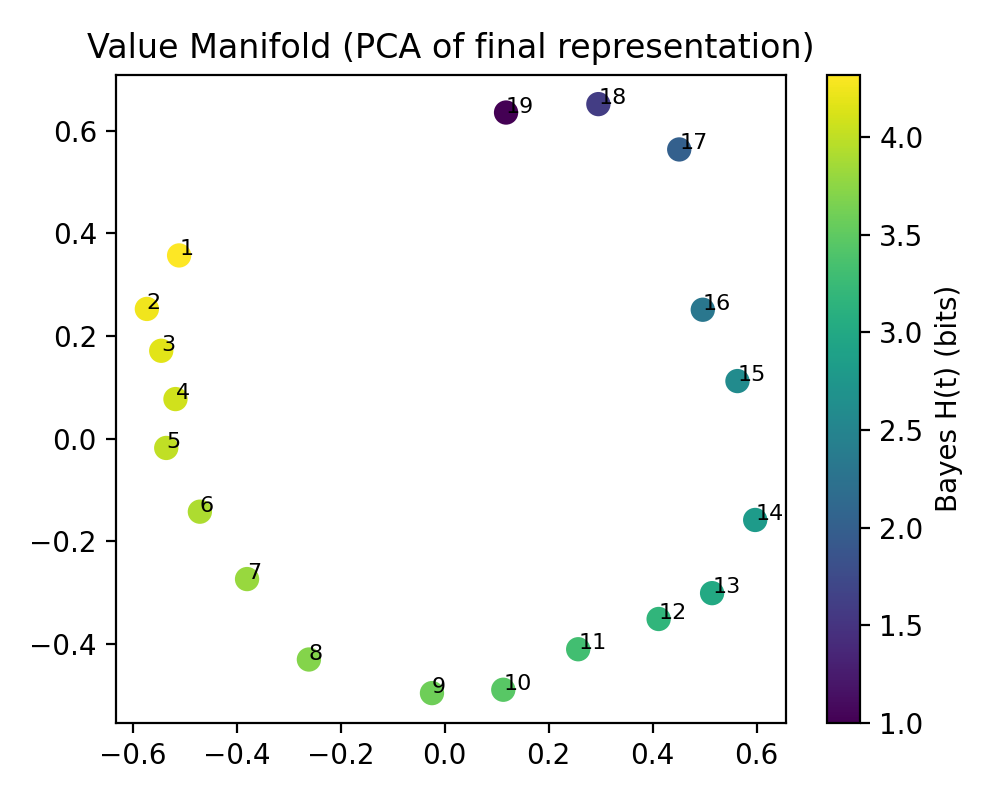}
    \caption{150k steps}
  \end{subfigure}
  \caption{\textbf{Value-manifold unfurling during training.}
  PCA projection of attention outputs in the bijection model, colored by analytic posterior entropy.
  At 100k steps, low-entropy states are tightly clustered; by 150k, they lie along a smooth one-dimensional curve parameterized by entropy, enabling fine-grained encoding of posterior states.Each point is an attention output (head output or block attention output -- whichever you used) at a supervised prediction position; PCA is fit on the pooled outputs and then plotted, colored by analytic posterior entropy.}
  \label{fig:value_manifold_comparison}
\end{figure}

\begin{figure}[htbp]
  \centering
  \includegraphics[width=0.7\textwidth]{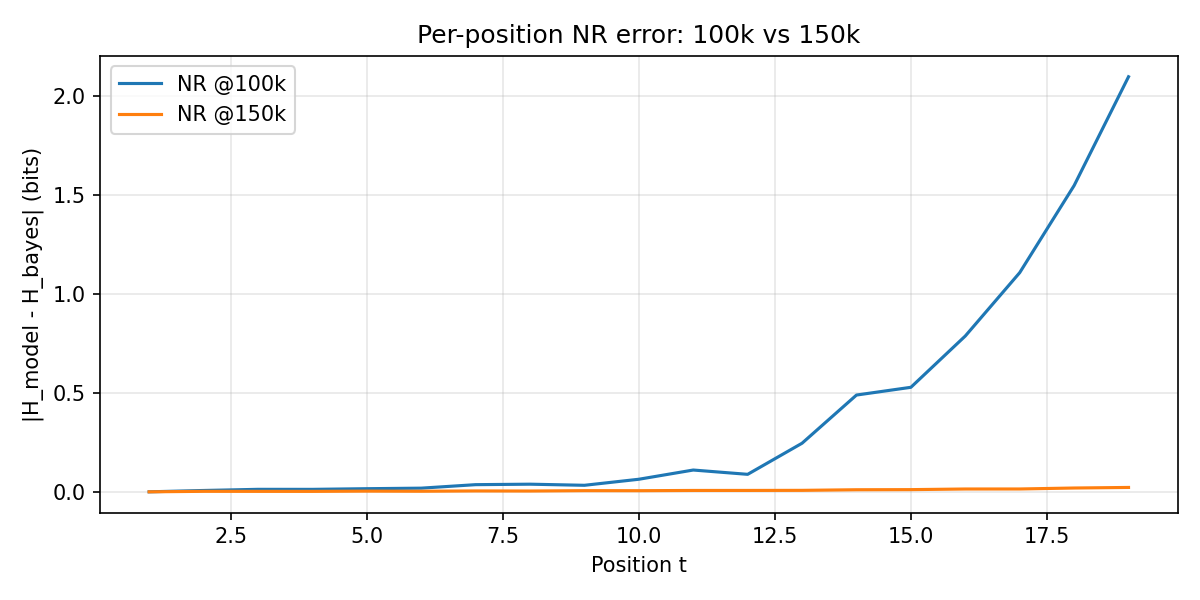}
  \caption{\textbf{Per-position calibration improves as the value manifold unfurls.}
  Absolute entropy error as a function of position in the bijection task at 100k and 150k training steps.
  The dominant improvements occur at late positions, matching the geometric unfurling of low-entropy states in \Cref{fig:value_manifold_comparison}.}
  \label{fig:perpos_error}
\end{figure}

\begin{figure}[t]
    \centering
    \begin{tikzpicture}[
        box/.style={
            draw,
            rounded corners=4pt,
            minimum width=7.5cm,
            minimum height=1.1cm,
            align=center,
            fill=gray!5
        },
        arrow/.style={->, thick},
        node distance=1.1cm
    ]

    \node[box] (binding) {
        \textbf{Layer 0: Foundational binding}\\
        Key--value hypothesis frame
    };

    \node[box, below=of binding] (elimination) {
        \textbf{Mid layers: Sequential elimination}\\
        Bayesian update in the residual stream
    };

    \node[box, below=of elimination] (refinement) {
        \textbf{Late layers: Manifold refinement}\\
        High-precision posterior encoding
    };

    \draw[arrow] (binding) -- (elimination);
    \draw[arrow] (elimination) -- (refinement);

    \end{tikzpicture}
    \caption{Three-stage architectural mechanism for Bayesian inference.
    Layer~0 constructs a key--value hypothesis frame, mid layers implement
    sequential Bayesian updates in the residual stream, and late layers refine
    the representation on a low-dimensional posterior manifold.}
    \label{fig:bayes_three_stage}
\end{figure}
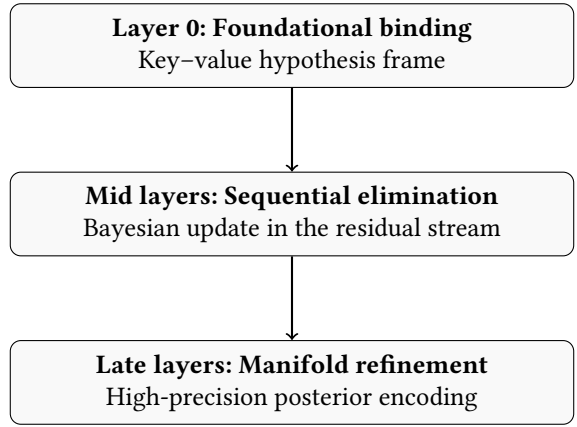

\section{Mechanism: How Transformers Realize Bayesian Inference}
\label{sec:mechanism}

The behavioral results in \Cref{sec:results} demonstrate that small transformers track
analytic Bayesian posteriors with sub-bit precision across two distinct wind-tunnel tasks.  We now
examine \emph{how} this computation is implemented internally.  Evidence from ablations, QK
geometry, probe dynamics, and training trajectories reveals a consistent architectural mechanism:
transformers perform Bayesian inference by constructing a representational frame, executing
sequential eliminations within that frame, and progressively refining posterior precision across
layers.

\subsection{Layer 0 Creates the Hypothesis Frame}

The computation begins with a structural operation: Layer~0 attention constructs the
\emph{hypothesis space} in which all subsequent inference takes place.  Keys at this layer form an
approximately orthogonal basis over input tokens (\Cref{fig:key_orthogonality}), providing a
coordinate system over which posterior mass can be represented and manipulated. We measure orthogonality via mean absolute off-diagonal cosine similarity between key vectors. Across 5 seeds, the bijection model achieves $0.052 \pm 0.004$ versus $0.082 \pm 0.003$ for random vectors in $d=192$ dimensions---a 37\% reduction ($p < 0.001$, paired $t$-test). The HMM model shows similar structure: $0.061 \pm 0.006$ vs.\ $0.079 \pm 0.002$ random baseline.

Head-wise ablations confirm the indispensability of this step. A single Layer~0 ``hypothesis-frame head''
dominates the layer's contribution (\Cref{fig:head_ablation}), and ablating this head alone
severely disrupts calibration. Here ``hypothesis-frame head'' means the head whose keys span the near-orthogonal
basis over hypothesis tokens and whose values instantiate the corresponding per-hypothesis slots in the residual stream.
No other attention head exhibits comparable sensitivity. This identifies a structural bottleneck: forming the
hypothesis frame is a prerequisite for any later Bayesian computation.

Once established, this frame remains stable through training.  Attention maps at Layer~0 change
little across checkpoints, even as the value manifold and calibration improve substantially.  The model therefore learns the geometry of the inference
problem early, and subsequently refines numerical precision within this fixed frame.

\subsection{Sequential Bayesian Elimination Across Depth}

With the hypothesis frame in place, the middle layers perform a layer-by-layer process that mirrors
Bayesian elimination.

\paragraph{Progressive QK sharpening.}
As depth increases, queries align more strongly with the subset of keys consistent with the observed
evidence (\Cref{fig:qk_alignment}).  Early layers attend broadly; deeper layers concentrate
attention almost exclusively on the feasible hypotheses.  This geometric focusing parallels analytic
Bayesian conditioning, where inconsistent hypotheses receive vanishing weight.

\paragraph{Hierarchical compositionality.}
Layer-wise ablations (\Cref{fig:layer0_vs_ffn}) show that removing any single layer
(attention + FFN, as implemented) increases calibration error by more than an order of magnitude.
This demonstrates that the computation is not shallow or redundant.  Each layer provides a distinct
and non-interchangeable refinement step, forming a sequential, compositional realization of Bayesian
updates.

Together, these observations indicate that transformers implement Bayesian elimination not via a
single transformation, but through a depth-wise sequence of projections and refinements within the
Layer~0 frame.

\subsection{Attention as Content-Addressable Routing}

Across all depths, attention serves a consistent geometric role: it retrieves the components of the
belief state relevant for the next update.

Three observations support this routing interpretation:

\begin{itemize}
    \item \textbf{Orthogonal keys} (\Cref{fig:key_orthogonality}) provide a basis for
    content-addressable lookup of hypotheses.
    \item \textbf{Sharpened QK alignment across depth} (\Cref{fig:qk_alignment}) routes
    residual-stream information toward the feasible hypothesis subspace.
    \item \textbf{Stable routing during late refinement} (\Cref{fig:value_manifold_comparison,fig:perpos_error})
    shows that once the frame is correct, attention maps change minimally even as calibration
    improves.
\end{itemize}

Routing is also essential for maintaining stable recursive inference.  In the HMM task, disabling
attention only in the top two layers leaves performance within the training horizon largely intact,
but long-horizon inference collapses (\Cref{fig:hmm_no_late_attn_scaling}).  Thus
attention is required both for forming the initial hypothesis frame and for sustaining stable belief
updates under extended rollout.

\subsection{Value-Space Manifolds and Precision Refinement}

After routing stabilizes, the final layers refine the \emph{precision} of the posterior representation.
\Cref{fig:value_manifold_comparison,fig:perpos_error} show that:

\begin{itemize}
    \item At intermediate checkpoints, value representations of low-entropy states are nearly
    collapsed and cannot reliably encode distinctions among small remaining hypothesis sets.
    \item By the final checkpoint, these states lie along a smooth \emph{one-dimensional manifold}
    parameterized by posterior entropy.
\end{itemize}

This geometric unfurling enables fine-grained encoding of posterior confidence and accounts for
late-position improvements in calibration.  This refinement occurs while attention maps
remain nearly unchanged, producing a clear \emph{frame--precision dissociation}: attention defines
where information flows, while downstream transformations refine how precisely beliefs are encoded.

\subsection{Synthesis: A Three-Stage Architectural Mechanism}

Across both wind tunnels, the evidence aligns into a three-stage mechanism (\Cref{fig:bayes_three_stage}):

\begin{enumerate}
    \item \textbf{Foundational binding (Layer 0).}  Construct an orthogonal hypothesis frame.
    (Key geometry; catastrophic Layer~0 head ablations.)
    \item \textbf{Progressive elimination (middle layers).}  Sequentially suppress inconsistent
    hypotheses through sharpening QK alignment.  (Layer-wise compositionality; geometric focusing.)
    \item \textbf{Precision refinement (late layers).}  Encode posterior entropy on a smooth
    value manifold while keeping routing fixed.  (Value-manifold unfurling; frame--precision
    dissociation.)
\end{enumerate}

This structure mirrors the analytic decomposition of Bayesian conditioning: define a hypothesis
space, update beliefs with evidence, and refine confidence as uncertainty decreases.

\subsection{Mamba's Mechanism: Selective State-Space Dynamics}
\label{sec:mamba-mechanism}

While the three-stage mechanism above is specific to transformer attention, Mamba achieves comparable HMM performance (0.024 vs 0.049 bits MAE) through qualitatively different mechanisms. Analyzing Mamba's internal representations reveals how selective state-space dynamics implement belief transport without attention.

\paragraph{Five-cluster geometry emerges from state selection.}
Mamba's final-layer representations organize into five discrete clusters corresponding to the five HMM hidden states (\Cref{fig:mamba_clusters}), with within-cluster variation encoding posterior entropy. This is the same corner geometry of the belief simplex that transformers learn, achieved through input-dependent state selection rather than query-key matching.

\paragraph{Distributed representations enable belief transport.}
Layer-wise analysis (\Cref{fig:hmm_layerwise_analysis}) reveals a key difference between Mamba and LSTM:
\begin{itemize}[leftmargin=1.2em]
    \item \textbf{Mamba}: Final-layer representations achieve $R^2 = 0.40$ for entropy prediction, with PC1 explaining only 21.9\% of variance---the model maintains a distributed, multi-dimensional belief representation.
    \item \textbf{LSTM}: Final-layer representations achieve $R^2 = 0.004$ (near random), with PC1 explaining 92.3\% of variance---the model collapses to a one-dimensional manifold uncorrelated with true posterior entropy.
\end{itemize}
This explains the primitives difference: under our training protocol, LSTM's fixed gating did not maintain the multi-dimensional structure required for belief transport, while Mamba's selection mechanism preserves it.

\paragraph{Layer-wise entropy encoding.}
Unlike transformers where entropy correlation increases sharply in later layers, Mamba shows moderate correlation ($|r| \approx 0.11$) across middle layers with gradual improvement in linear predictability ($R^2$) from 0.025 at Layer 1 to 0.40 at Layer 9. This suggests Mamba accumulates information more gradually through its recurrent state rather than in discrete elimination steps.

\paragraph{Convergence to Bayesian geometry.}
The fact that different mechanisms---attention's query-key matching and Mamba's selective state updates---converge to similar geometric solutions (corner geometry of the belief simplex) suggests that Bayesian geometry may be a universal attractor for architectures with content-based routing. This provides a geometric explanation for why both architectures succeed at belief transport despite using fundamentally different computational primitives.

\begin{figure}[htbp]
  \centering
  \includegraphics[width=0.9\textwidth]{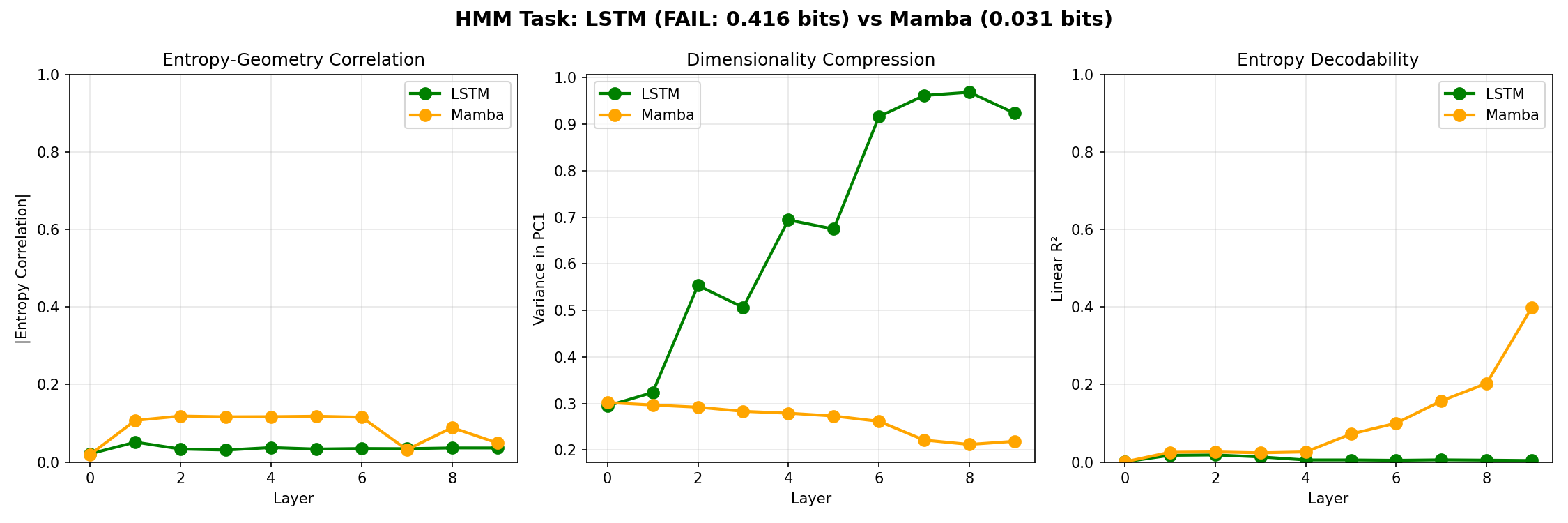}
  \caption{\textbf{Layer-wise analysis: Mamba vs LSTM on HMM.}
  For each layer, we measure PC1 variance (dimensionality), entropy correlation (alignment with uncertainty), and linear $R^2$ (predictability of entropy).
  \emph{Left:} LSTM collapses to a 1D manifold (PC1 $>$ 90\%) with near-zero entropy correlation.
  \emph{Right:} Mamba maintains distributed representations (PC1 $\approx$ 22\%) with increasing entropy predictability across depth.
  This explains why LSTM fails at belief transport while Mamba succeeds.}
  \label{fig:hmm_layerwise_analysis}
\end{figure}

\subsection{Relation to Gradient-Dynamics Predictions}

These empirical observations match predictions from recent analyses of gradient dynamics, which
show that attention scores tend to stabilize once the correct routing structure has formed, while
value and residual representations continue to refine precision.  The observed stability of attention
maps together with the unfolding of the value manifold provides direct evidence for this
\emph{differential convergence} of routing and precision.

\section{Analysis and Discussion}
\label{sec:discussion}

The wind-tunnel experiments demonstrate that small transformers, trained with standard
optimization and without architectural modifications, implement Bayesian inference with
striking fidelity.  In this section we discuss the broader implications of these results for
interpretability, architectural necessity, and the connection between controlled wind tunnels
and the behavior of large language models.

\subsection{Why Hierarchical Attention Implements Bayes}

Across the bijection and HMM settings, the internal geometry uncovered in \Cref{sec:mechanism}
reveals a consistent computational pattern.  Transformers realize Bayesian conditioning through a
stacked sequence of geometric operations:

\begin{enumerate}
    \item \textbf{Foundational binding (Layer~0).}  
    Orthogonal keys create a hypothesis frame.  The catastrophic effect of ablating the Layer~0
    hypothesis-frame head (\Cref{fig:head_ablation}) demonstrates that this frame is structurally
    indispensable.

    \item \textbf{Progressive elimination (middle layers).}  
    QK-alignment sharpens across depth (\Cref{fig:qk_alignment}), mirroring the multiplicative
    suppression of ruled-out hypotheses in analytic Bayesian updates.  
    Layer-wise ablations (\Cref{fig:layer0_vs_ffn}) show that each layer contributes a
    non-interchangeable refinement step.

    \item \textbf{Precision refinement (late layers).}  
    Once routing stabilizes, value representations unfold into a low-dimensional manifold
    parameterized by posterior entropy (\Cref{fig:value_manifold_comparison}), improving calibration
    particularly at late positions (\Cref{fig:perpos_error}).  
    This frame--precision dissociation reflects a division of labor: attention establishes where
    information flows, while subsequent transformations refine the numerical precision of the belief.
\end{enumerate}

This hierarchy parallels Bayes' rule: define a hypothesis space, integrate evidence, and refine the
posterior.  The transformer implements these steps using attention geometry and residual-stream
representations.

\subsection{Depth as Compositional Necessity}

A central conclusion from the ablation studies is that depth is not redundant.  In both wind tunnels,
removing any individual layer increases calibration error by more than an order of magnitude
(\Cref{fig:layer0_vs_ffn}).  
This shows that Bayesian reasoning is expressed as a sequence of compositional projections, each
layer refining the belief state in a way that cannot be collapsed into a single transformation.

This stands in contrast to wide, shallow architectures: even with comparable parameter counts and
identical training, MLPs fail to perform hypothesis elimination or state tracking (\Cref{sec:results-arch}).
Bayesian inference requires \emph{hierarchical refinement}, and transformers supply the appropriate
inductive bias through depth and residual composition.

\subsection{From Wind Tunnels to Natural Language}

While the wind tunnels are deliberately simplified, they capture the essential structure of
probabilistic inference: integrating evidence over time to update latent beliefs.  
Large language models operate in a far more complex setting, with high-dimensional latent spaces
and ambiguous, multi-modal evidence.  Yet the geometric ingredients observed here---orthogonal
hypothesis axes, depth-wise refinement, and stable routing---are structural rather than
task-specific.

The results therefore suggest that the probabilistic behaviors exhibited by LLMs may arise not only
from scale or data richness but also from architectural geometry.  
Wind tunnels provide a verifiable lower bound: they show that transformers \emph{can} implement
Bayesian inference exactly when the posterior is known.

\subsection{The Primitives Taxonomy Explains Architectural Differences}

The four-architecture comparison (\Cref{tab:arch_comparison}) reveals that Bayesian inference is not monolithic---different tasks demand different inference primitives, and different architectures realize different subsets.

\paragraph{Why Mamba beats Transformer on HMM but loses on recall.}
Mamba's selective state-space mechanism implements content-based routing on \emph{transition dynamics}: the input-dependent matrices $(\Delta, B, C)$ control what information propagates forward through the recurrence. This is ideal for belief transport---tracking how a posterior evolves through stochastic dynamics---which explains Mamba's superiority on HMM (0.024 vs 0.049 bits).

But binding requires a different operation: given a query, retrieve the associated value from memory. Attention provides this directly via query-key matching with $O(1)$ access to any position. Mamba must simulate retrieval through its recurrent state, which is slower and less precise. This explains the 2.5$\times$ longer training and 97.8\% vs 100\% accuracy gap on associative recall.

\paragraph{Why LSTM succeeds on bijection but fails everywhere else.}
LSTM's gates depend only on $(h_{t-1}, x_t)$---under our training protocol, they did not learn to perform content-based matching across positions. This suffices for accumulating \emph{static} sufficient statistics (bijection admits a fixed-dimensional statistic: the set of observed outputs), but fails when:
\begin{itemize}
    \item \textbf{The statistic must evolve under dynamics (transport)}: LSTM compresses HMM representations to a 1D manifold uncorrelated with the true 5D belief state---it did not transport the belief vector through the transition matrix
    \item \textbf{The statistic must be indexed by content (binding)}: LSTM achieves 0.5\% on recall (random chance)---it did not retrieve by content
\end{itemize}

\paragraph{Implications.}
The primitives framework provides a principled basis for architecture selection: match the architecture's capabilities to the task's primitive requirements. It also explains why attention remains essential for tasks requiring flexible retrieval, even as alternatives like Mamba excel at sequence modeling.

\subsection{A Lower Bound for Reasoning in LLMs}

The wind tunnels establish a principled baseline for mechanistic reasoning in transformers.
If a model cannot implement Bayes in a setting with a closed-form posterior and impossible
memorization, it offers little evidence of genuine inference capability in natural language.  
Conversely, the fact that small, verifiable transformers succeed here---with interpretable geometric
mechanisms---suggests that similar structures may underpin reasoning in large models.

This provides a concrete research direction: search for the same geometric signatures in frontier
LLMs.  
The diagnostics used here---key orthogonality, QK sharpening, value-manifold structure, and
routing stability---offer testable predictions for analyzing pretrained language models.

\subsection{A Dual-Entropy Measurement Framework}
\label{sec:dual-entropy}

The wind tunnels provide a setting where inference quality can be measured with bit-level precision.
We formalize this with a dual-entropy framework that separates two quantities often conflated in practice.

\begin{definition}[Dual Entropies]
\label{def:dual-entropy}
For a context sequence $s$:
\begin{itemize}
    \item \textbf{Context surprisal} $H_I(s) = -\log p_{\mathrm{train}}(s)$ measures the \emph{distinctiveness} of $s$ in the training distribution (high $H_I$ = rare, informative context).
    \item \textbf{Prediction entropy} $H_P(s) = -\sum_i p_{\mathrm{model}}(t_i \mid s) \log p_{\mathrm{model}}(t_i \mid s)$ measures the model's \emph{predictive uncertainty} conditioned on $s$ (low $H_P$ = confident prediction).
\end{itemize}
\end{definition}

The ratio $\rho = H_P / H_I$ is a normalized confidence-per-information coefficient.
Low $\rho$ from high $H_I$ indicates that the model leveraged distinctive contextual evidence to concentrate its posterior---compositional inference.
Low $\rho$ from low $H_I$ indicates retrieval from frequent training patterns---rote memorization.
Prediction entropy $H_P$ alone cannot distinguish these regimes.

\paragraph{Why wind tunnels are essential.}
For natural language, $H_I$ requires knowing the training distribution, which is unavailable.
In the wind tunnel, both quantities are analytically computable: the Bayesian posterior entropy plays the role of $H_I$---it measures how informative the observed sequence is given the hypothesis class---and the model's predictive entropy is $H_P$.
The entropy MAE reported throughout this paper is $|H_P - H_{\mathrm{Bayes}}|$: the gap between the model's confidence and the confidence warranted by the evidence.

In our bijection wind tunnels, trained transformers achieve $\rho \to 0$: prediction entropy tracks Bayesian posterior entropy to $0.007$ bits, meaning the model's confidence is precisely calibrated to the information content at every position.
This validates the framework quantitatively: the geometric structure described in \Cref{sec:mechanism}---orthogonal hypothesis frames, entropy-parameterized value manifolds, and stable routing---is the computational substrate that produces low $\rho$.
Papers~II--III show that the gradient dynamics producing this geometry implement an EM-like algorithm where the advantage signal driving routing is large precisely when $H_I$ is high and $H_P$ is low~\citep{dalal2025gradient,aggarwal2025geometric3}, and that the same low-$\rho$ geometry persists in production language models.

\section{Related Work}
\label{sec:related}

\subsection{Bayesian Interpretations of Deep Learning}
A long line of work interprets neural networks through a Bayesian lens, 
from classical analyses of predictive uncertainty
\citep{mackay1992bayesian,neal2012bayesian}
to variational or stochastic approximations of posterior inference
\citep{graves2011practical,blundell2015weight}.
Recent papers argue that, in large-data limits, 
minimizing cross-entropy implicitly targets the Bayesian posterior predictive 
\citep{xie2022explanation,vonoswald2023transformers}.
These results concern what training \emph{should} produce at the population level.
Our contribution is complementary: a controlled setting in which the true posterior is known,
memorization is computationally infeasible, and one can directly test whether a finite transformer 
\emph{actually} realizes this Bayesian computation.

\subsection{In-Context Learning and Algorithmic Generalization}
Transformers have been shown to perform algorithmic tasks in context,
including arithmetic \citep{garg2022can},
synthetic induction \citep{elman1990finding},
and more general pattern extrapolation \citep{akyurek2023what,olsson2022context}.
Behaviorally, these models often resemble Bayesian learners,
an observation formalized by recent explanatory theories
\citep{xie2022explanation,vonoswald2023transformers}.
\citet{muller2024context} provide evidence that high-capacity transformers often mimic the Bayesian predictor during in-context learning, and \citet{reuter2025transformers} demonstrate that transformers can perform full Bayesian inference for statistical models including generalized linear models and latent factor models, achieving results comparable to expensive exact methods.

However, prior work cannot distinguish true Bayesian computation
from learned heuristics or memorized templates,
because the ground-truth posterior is unknown for natural language tasks.
Furthermore, these studies focus on \emph{whether} transformers can implement Bayesian inference, not \emph{why} they succeed where other architectures fail.

Our wind-tunnel methodology addresses both gaps:
by constructing tasks with closed-form analytic posteriors and combinatorially large hypothesis spaces,
we obtain a direct pointwise comparison between model predictions and Bayes' rule.
By comparing multiple architectures, we identify the \emph{inference primitives} that explain success and failure.
This moves the discussion from correlation to mechanism, and from existence to architectural characterization.

\subsection{Mechanistic Interpretability and Attention Geometry}
Mechanistic studies of transformers have revealed specialized attention heads for induction,
copying, and retrieval \citep{elhage2021mathematical,nanda2023progress}.
Other work has examined QKV spaces, circuit decomposition,
and sparse structures that arise during training \citep{olsson2022context}.
These studies provide qualitative and circuit-level insight into model behaviors.

Our contribution is to link these geometric structures directly to 
\emph{Bayesian inference} in a setting where the posterior is known.  
We show that keys form near-orthogonal hypothesis axes, 
queries sharpen onto feasible hypotheses across depth, 
and value representations unfurl into a one-dimensional entropy manifold.
This connects mechanistic interpretability to probabilistic computation in a rigorous way:
the internal geometry needed for Bayesian reasoning becomes directly visible.

\subsection{Architectural Comparisons and the Copying Problem}
Alternative sequence models---state-space architectures \citep{gu2022efficiently,gu2024mamba},
convolutional variants \citep{poli2023hyena}, and recurrent networks---
often match transformers in perplexity on natural text.
But perplexity conflates modeling and inference capability.
Our results provide a finer test: whether an architecture
can reproduce an analytic Bayesian posterior under strict non-memorization constraints.

Recent work has identified a fundamental limitation of state-space models: \citet{jelassi2024repeat} show that SSMs struggle with copying and retrieval tasks because they compress input into a fixed-size latent state. Transformers with 10$\times$ fewer parameters outperform Mamba on retrieval, and Mamba requires 100$\times$ more training data to learn simple copying. This limitation stems from the recurrent structure itself---information must be explicitly stored in state rather than retrieved on demand.

Our primitives framework provides a unifying explanation for these findings. We decompose Bayesian inference into three primitives: \emph{belief accumulation} (integrating evidence), \emph{belief transport} (propagating beliefs through dynamics), and \emph{random-access binding} (retrieving hypotheses by content). The copying/retrieval limitation identified by \citet{jelassi2024repeat} corresponds precisely to our binding primitive. Mamba's selective state-space mechanism \citep{gu2024mamba}---where the transition matrices $\Delta$, $B$, $C$ become input-dependent---implements content-based routing on the \emph{transition dynamics}, enabling accumulation and transport. But this is fundamentally different from attention's ability to directly retrieve arbitrary past positions via query-key matching.

This explains our empirical pattern: Mamba \emph{outperforms} transformers on HMM filtering (0.027 vs 0.049 bits MAE), a task dominated by belief transport, because its selection mechanism is optimized for controlling what information propagates forward. But Mamba is slower on associative recall (97.8\% vs 100\%, requiring 2.5$\times$ more epochs), a task requiring random-access binding. LSTMs fail on both because their gates depend only on $(h_{t-1}, x_t)$---they did not learn content-based matching across positions under our training protocol. The primitives taxonomy thus predicts which architecture will excel on which task, resolving apparent contradictions in the literature.

Recent work shows SSM performance depends strongly on optimization: mimetic initialization can yield attention-like behavior in Mamba, improving copying and autoregressive tasks. Our primitives conclusions characterize what each architecture \emph{achieved} under standard training, not absolute capability limits. With specialized initialization or broader hyperparameter sweeps, the gaps we observe might narrow---though the qualitative pattern (Mamba excelling at transport, struggling with binding) appears robust across our experiments.

\subsection{Training Dynamics}
Finally, concurrent work analyzes the gradient dynamics that create these structures during training
\citep{dalal2025gradient}.
They show that attention and value updates follow coupled laws that produce 
a stable routing frame and a progressively refined value manifold.  
Our empirical findings align with this picture: 
attention stabilizes early, while value vectors continue to encode the posterior 
with increasing resolution.
Together, these perspectives connect the optimization trajectory to 
the geometric structure that implements Bayesian inference.
\section{Limitations and Future Work}
\label{sec:limitations}

Our experiments are intentionally small-scale: they use controlled Bayesian wind tunnels with 
analytic posteriors, modest vocabulary sizes, and transformers with 2--3M parameters.  
This regime is what makes mechanistic verification possible, but it naturally abstracts away 
from the full complexity of natural-language inference.  
Several limitations therefore remain, which point directly toward future extensions.

\paragraph{Scale and richness of inference tasks.}
Bijections and HMMs capture essential elements of Bayesian computation--discrete elimination and 
recursive state tracking--but they represent only a narrow slice of the inference problems 
encountered by large language models.  
Future wind tunnels could incorporate richer latent-variable structures, including
Kalman filtering, hierarchical Bayesian models, or causal graphical models,
all of which have closed-form posteriors and allow precise verification.

\paragraph{Dimensionality of hypothesis spaces.}
Although the hypothesis spaces in both tasks are large enough to prevent memorization,
their representational dimensionality is modest (e.g., five hidden states in HMMs).
Larger systems with high-dimensional latent variables would test whether the geometric 
mechanisms we observe--orthogonal hypothesis axes, progressive Q--K sharpening, and 
value-manifold refinement--scale smoothly with dimensionality.

\paragraph{Connection to large pretrained models.}
Our geometric diagnostics (key orthogonality, score-gradient structure, value manifolds) 
are testable predictions for frontier LLMs.  
Whether similar Bayesian manifolds arise in large models trained on natural text remains 
an open question.  
Applying these tools directly to pretrained transformer layers is a natural next step and 
may reveal how approximate Bayesian structure manifests in more complex settings.

\paragraph{Architectural generality.}
The experiments here use standard transformers.  
It remains unclear whether alternative architectures---state-space models, deep MLPs with
more sophisticated gating, or hybrid recurrent-attention systems---can form comparable 
Bayesian manifolds.  
Wind-tunnel evaluations could provide a principled benchmark for comparing architectures 
in terms of inference fidelity rather than perplexity alone.

\paragraph{Training dynamics and phase transitions.}
A notable empirical phenomenon is the frame--precision dissociation:
attention maps stabilize early while value manifolds continue to unfurl and refine posterior 
precision.  
A systematic study of these phases---how early the frame forms, how quickly precision improves,
and how these dynamics depend on depth, width, and data complexity---could lead to a more 
general theory of representation formation in transformers.

\paragraph{Towards natural-language wind tunnels.}
Ultimately, we aim to understand how the exact Bayesian reasoning demonstrated here relates 
to the approximate reasoning observed in natural language tasks.  
Wind tunnels provide a lower bound: they establish that transformers \emph{can} implement 
Bayesian updates when the problem is well specified.  
The next challenge is to design controlled tasks embedded within naturalistic language data 
that preserve analytic structure while introducing real-world ambiguity.

\section{Conclusion}
\label{sec:conclusion}

We introduced Bayesian wind tunnels---controlled experimental settings with analytic posteriors
and combinatorially large hypothesis spaces---to test whether neural sequence models genuinely implement
Bayesian inference rather than merely mimicking it. Across multiple inference problems, small transformers converge to the exact
Bayesian posterior with sub-bit calibration error, even at sequence lengths well beyond those
seen in training.

The key insight is that Bayesian inference is not monolithic. We decompose it into three \emph{inference primitives}---belief accumulation, belief transport, and random-access binding---and show that different architectures realize different subsets:
\begin{itemize}[itemsep=2pt]
    \item \textbf{Transformers} realize all three primitives and succeed uniformly.
    \item \textbf{Mamba} realizes accumulation and transport, achieving state-of-the-art on HMM filtering, but struggles with random-access binding.
    \item \textbf{LSTMs} realize only accumulation of static sufficient statistics: they succeed on belief revision (where the statistic is fixed-dimensional) but fail when the statistic must evolve under dynamics or be indexed by content.
    \item \textbf{MLPs} realize none and fail uniformly.
\end{itemize}

Geometric diagnostics reveal how these primitives are implemented.
Keys form an approximately orthogonal basis over hypotheses; queries progressively align with
the feasible region of that basis; and value vectors organize along a low-dimensional manifold
parameterized by posterior entropy.
On HMM tracking, Mamba's representations organize into five discrete clusters---one per hidden state---showing that the model discovers the corner geometry of the belief simplex.
Training sculpts this manifold: attention patterns stabilize early, while value representations
continue refining posterior precision---a frame--precision dissociation predicted by concurrent
gradient-dynamics analysis.

The wind-tunnel regime is intentionally simplified, but it establishes a clear lower bound:
if a model cannot implement Bayes in settings where the posterior is known and memorization
is impossible, it cannot do so in natural language.
Conversely, our results show that content-based routing is sufficient for exact Bayesian inference
when the task demands only the primitives that architecture can realize.
This provides a principled foundation for studying approximate reasoning in larger models and
offers concrete, testable predictions---orthogonal hypothesis axes, progressive Q--K sharpening,
and value-manifold structure---for analysing pretrained LLMs.

The primitives framework explains why transformers succeed: they furnish
the architectural mechanisms for all three inference primitives. Attention provides random-access binding through query-key matching; content-based routing enables belief transport; and the residual stream accumulates evidence across positions.
Understanding how these primitives scale to real-world language, and whether new architectures can realize them more efficiently, remains an important direction for future work.

\bibliographystyle{ACM-Reference-Format}
\bibliography{references}


\begin{thebibliography}{20}


\ifx \showCODEN    \undefined \def \showCODEN     #1{\unskip}     \fi
\ifx \showISBNx    \undefined \def \showISBNx     #1{\unskip}     \fi
\ifx \showISBNxiii \undefined \def \showISBNxiii  #1{\unskip}     \fi
\ifx \showISSN     \undefined \def \showISSN      #1{\unskip}     \fi
\ifx \showLCCN     \undefined \def \showLCCN      #1{\unskip}     \fi
\ifx \shownote     \undefined \def \shownote      #1{#1}          \fi
\ifx \showarticletitle \undefined \def \showarticletitle #1{#1}   \fi
\ifx \showURL      \undefined \def \showURL       {\relax}        \fi
\providecommand\bibfield[2]{#2}
\providecommand\bibinfo[2]{#2}
\providecommand\natexlab[1]{#1}
\providecommand\showeprint[2][]{arXiv:#2}

\bibitem[Agarwal et~al\mbox{.}(2025a)]%
        {aggarwal2025geometric3}
\bibfield{author}{\bibinfo{person}{Naman Agarwal},
  \bibinfo{person}{Siddhartha~R. Dalal}, {and} \bibinfo{person}{Vishal Misra}.}
  \bibinfo{year}{2025}\natexlab{a}.
\newblock \bibinfo{title}{Geometric Scaling of Bayesian Inference in LLMs}.
\newblock
\showeprint[arxiv]{2512.23752}~[cs.CL]
\urldef\tempurl%
\url{https://arxiv.org/abs/2512.23752}
\showURL{%
\tempurl}
\newblock
\shownote{Paper III of the Bayesian Attention Trilogy}.


\bibitem[Agarwal et~al\mbox{.}(2025b)]%
        {dalal2025gradient}
\bibfield{author}{\bibinfo{person}{Naman Agarwal},
  \bibinfo{person}{Siddhartha~R. Dalal}, {and} \bibinfo{person}{Vishal Misra}.}
  \bibinfo{year}{2025}\natexlab{b}.
\newblock \bibinfo{title}{Gradient Dynamics of Attention: How Cross-Entropy
  Sculpts Bayesian Manifolds}.
\newblock
\showeprint[arxiv]{2512.22473}~[cs.LG]
\urldef\tempurl%
\url{https://arxiv.org/abs/2512.22473}
\showURL{%
\tempurl}
\newblock
\shownote{Paper II of the Bayesian Attention Trilogy}.


\bibitem[Aky{\"u}rek et~al\mbox{.}(2023)]%
        {akyurek2023what}
\bibfield{author}{\bibinfo{person}{Ekin Aky{\"u}rek}, \bibinfo{person}{Dale
  Schuurmans}, \bibinfo{person}{Jacob Andreas}, \bibinfo{person}{Tengyu Ma},
  {and} \bibinfo{person}{Denny Zhou}.} \bibinfo{year}{2023}\natexlab{}.
\newblock \showarticletitle{What Learning Algorithm is In-Context Learning?
  Investigations with Linear Models}. In
  \bibinfo{booktitle}{\emph{International Conference on Learning
  Representations}}.
\newblock


\bibitem[Blundell et~al\mbox{.}(2015)]%
        {blundell2015weight}
\bibfield{author}{\bibinfo{person}{Charles Blundell}, \bibinfo{person}{Julien
  Cornebise}, \bibinfo{person}{Koray Kavukcuoglu}, {and} \bibinfo{person}{Daan
  Wierstra}.} \bibinfo{year}{2015}\natexlab{}.
\newblock \showarticletitle{Weight Uncertainty in Neural Networks}. In
  \bibinfo{booktitle}{\emph{Proceedings of the 32nd International Conference on
  Machine Learning}}. \bibinfo{pages}{1613--1622}.
\newblock


\bibitem[Elhage et~al\mbox{.}(2021)]%
        {elhage2021mathematical}
\bibfield{author}{\bibinfo{person}{Nelson Elhage}, \bibinfo{person}{Neel
  Nanda}, \bibinfo{person}{Catherine Olsson}, \bibinfo{person}{Tom Henighan},
  \bibinfo{person}{Nicholas Joseph}, \bibinfo{person}{Ben Mann},
  \bibinfo{person}{Amanda Askell}, \bibinfo{person}{Yuntao Bai},
  \bibinfo{person}{Anna Chen}, \bibinfo{person}{Tom Conerly},
  \bibinfo{person}{Nova DasSarma}, \bibinfo{person}{Dawn Drain},
  \bibinfo{person}{Deep Ganguli}, \bibinfo{person}{Zac Hatfield-Dodds},
  \bibinfo{person}{Danny Hernandez}, \bibinfo{person}{Andy Jones},
  \bibinfo{person}{Jackson Kernion}, \bibinfo{person}{Liane Lovitt},
  \bibinfo{person}{Kamal Ndousse}, \bibinfo{person}{Dario Amodei},
  \bibinfo{person}{Tom Brown}, \bibinfo{person}{Jack Clark},
  \bibinfo{person}{Jared Kaplan}, \bibinfo{person}{Sam McCandlish}, {and}
  \bibinfo{person}{Chris Olah}.} \bibinfo{year}{2021}\natexlab{}.
\newblock \bibinfo{title}{A Mathematical Framework for Transformer Circuits}.
\newblock \bibinfo{howpublished}{Transformer Circuits Thread, Anthropic}.
\newblock
\urldef\tempurl%
\url{https://transformer-circuits.pub/2021/framework/index.html}
\showURL{%
\tempurl}


\bibitem[Elman(1990)]%
        {elman1990finding}
\bibfield{author}{\bibinfo{person}{Jeffrey~L Elman}.}
  \bibinfo{year}{1990}\natexlab{}.
\newblock \showarticletitle{Finding structure in time}.
\newblock \bibinfo{journal}{\emph{Cognitive Science}} \bibinfo{volume}{14},
  \bibinfo{number}{2} (\bibinfo{year}{1990}), \bibinfo{pages}{179--211}.
\newblock


\bibitem[Garg et~al\mbox{.}(2022)]%
        {garg2022can}
\bibfield{author}{\bibinfo{person}{Shivam Garg}, \bibinfo{person}{Dimitris
  Tsipras}, \bibinfo{person}{Percy~S. Liang}, {and} \bibinfo{person}{Gregory
  Valiant}.} \bibinfo{year}{2022}\natexlab{}.
\newblock \showarticletitle{What Can Transformers Learn In-Context? {A} Case
  Study of Simple Function Classes}. In \bibinfo{booktitle}{\emph{Advances in
  Neural Information Processing Systems}}, Vol.~\bibinfo{volume}{35}.
  \bibinfo{pages}{29881--29895}.
\newblock


\bibitem[Graves(2011)]%
        {graves2011practical}
\bibfield{author}{\bibinfo{person}{Alex Graves}.}
  \bibinfo{year}{2011}\natexlab{}.
\newblock \showarticletitle{Practical Variational Inference for Neural
  Networks}. In \bibinfo{booktitle}{\emph{Advances in Neural Information
  Processing Systems}}, Vol.~\bibinfo{volume}{24}.
\newblock


\bibitem[Gu and Dao(2023)]%
        {gu2024mamba}
\bibfield{author}{\bibinfo{person}{Albert Gu} {and} \bibinfo{person}{Tri Dao}.}
  \bibinfo{year}{2023}\natexlab{}.
\newblock \showarticletitle{Mamba: Linear-Time Sequence Modeling with Selective
  State Spaces}.
\newblock \bibinfo{journal}{\emph{arXiv preprint arXiv:2312.00752}}
  (\bibinfo{year}{2023}).
\newblock


\bibitem[Gu et~al\mbox{.}(2022)]%
        {gu2022efficiently}
\bibfield{author}{\bibinfo{person}{Albert Gu}, \bibinfo{person}{Karan Goel},
  {and} \bibinfo{person}{Christopher R{\'e}}.} \bibinfo{year}{2022}\natexlab{}.
\newblock \showarticletitle{Efficiently Modeling Long Sequences with Structured
  State Spaces}. In \bibinfo{booktitle}{\emph{International Conference on
  Learning Representations}}.
\newblock


\bibitem[Jelassi et~al\mbox{.}(2024)]%
        {jelassi2024repeat}
\bibfield{author}{\bibinfo{person}{Samy Jelassi}, \bibinfo{person}{David
  Brandfonbrener}, \bibinfo{person}{Sham~M. Kakade}, {and}
  \bibinfo{person}{Eran Malach}.} \bibinfo{year}{2024}\natexlab{}.
\newblock \showarticletitle{Repeat After Me: Transformers are Better than State
  Space Models at Copying}.
\newblock \bibinfo{journal}{\emph{arXiv preprint arXiv:2402.01032}}
  (\bibinfo{year}{2024}).
\newblock
\urldef\tempurl%
\url{https://arxiv.org/abs/2402.01032}
\showURL{%
\tempurl}


\bibitem[MacKay(1992)]%
        {mackay1992bayesian}
\bibfield{author}{\bibinfo{person}{David J.~C. MacKay}.}
  \bibinfo{year}{1992}\natexlab{}.
\newblock \showarticletitle{A Practical {B}ayesian Framework for
  Backpropagation Networks}.
\newblock \bibinfo{journal}{\emph{Neural Computation}} \bibinfo{volume}{4},
  \bibinfo{number}{3} (\bibinfo{year}{1992}), \bibinfo{pages}{448--472}.
\newblock


\bibitem[Nanda et~al\mbox{.}(2023)]%
        {nanda2023progress}
\bibfield{author}{\bibinfo{person}{Neel Nanda}, \bibinfo{person}{Lawrence
  Chan}, \bibinfo{person}{Tom Lieberum}, \bibinfo{person}{Jess Smith}, {and}
  \bibinfo{person}{Jacob Steinhardt}.} \bibinfo{year}{2023}\natexlab{}.
\newblock \showarticletitle{Progress Measures for Grokking via Mechanistic
  Interpretability}.
\newblock \bibinfo{journal}{\emph{arXiv preprint arXiv:2301.05217}}
  (\bibinfo{year}{2023}).
\newblock


\bibitem[Neal(2012)]%
        {neal2012bayesian}
\bibfield{author}{\bibinfo{person}{Radford~M. Neal}.}
  \bibinfo{year}{2012}\natexlab{}.
\newblock \bibinfo{booktitle}{\emph{{B}ayesian Learning for Neural Networks}}.
  \bibinfo{series}{Lecture Notes in Statistics}, Vol.~\bibinfo{volume}{118}.
\newblock \bibinfo{publisher}{Springer}.
\newblock


\bibitem[Olsson et~al\mbox{.}(2022)]%
        {olsson2022context}
\bibfield{author}{\bibinfo{person}{Catherine Olsson}, \bibinfo{person}{Nelson
  Elhage}, \bibinfo{person}{Neel Nanda}, \bibinfo{person}{Nicholas Joseph},
  \bibinfo{person}{Nova DasSarma}, \bibinfo{person}{Tom Henighan},
  \bibinfo{person}{Ben Mann}, \bibinfo{person}{Amanda Askell},
  \bibinfo{person}{Yuntao Bai}, \bibinfo{person}{Anna Chen},
  \bibinfo{person}{Tom Conerly}, \bibinfo{person}{Dawn Drain},
  \bibinfo{person}{Deep Ganguli}, \bibinfo{person}{Zac Hatfield-Dodds},
  \bibinfo{person}{Danny Hernandez}, {et~al\mbox{.}}}
  \bibinfo{year}{2022}\natexlab{}.
\newblock \bibinfo{title}{In-Context Learning and Induction Heads}.
\newblock \bibinfo{howpublished}{Transformer Circuits Thread, Anthropic}.
\newblock
\urldef\tempurl%
\url{https://transformer-circuits.pub/2022/in-context-learning-and-induction-heads/index.html}
\showURL{%
\tempurl}


\bibitem[Panwar et~al\mbox{.}(2024)]%
        {muller2024context}
\bibfield{author}{\bibinfo{person}{Madhur Panwar}, \bibinfo{person}{Kabir
  Ahuja}, {and} \bibinfo{person}{Navin Goyal}.}
  \bibinfo{year}{2024}\natexlab{}.
\newblock \showarticletitle{In-Context Learning Through the Bayesian Prism}. In
  \bibinfo{booktitle}{\emph{International Conference on Learning
  Representations}}.
\newblock
\urldef\tempurl%
\url{https://openreview.net/forum?id=HX5ujdsSon}
\showURL{%
\tempurl}


\bibitem[Poli et~al\mbox{.}(2023)]%
        {poli2023hyena}
\bibfield{author}{\bibinfo{person}{Michael Poli}, \bibinfo{person}{Stefano
  Massaroli}, {et~al\mbox{.}}} \bibinfo{year}{2023}\natexlab{}.
\newblock \showarticletitle{Hyena Hierarchy: Towards Larger Convolutional
  Language Models}. In \bibinfo{booktitle}{\emph{International Conference on
  Machine Learning}}.
\newblock


\bibitem[Reuter et~al\mbox{.}(2025)]%
        {reuter2025transformers}
\bibfield{author}{\bibinfo{person}{Arik Reuter}, \bibinfo{person}{Tim G.~J.
  Rudner}, \bibinfo{person}{Vincent Fortuin}, {and} \bibinfo{person}{David
  R{\"u}gamer}.} \bibinfo{year}{2025}\natexlab{}.
\newblock \showarticletitle{Can Transformers Learn Full Bayesian Inference in
  Context?}. In \bibinfo{booktitle}{\emph{International Conference on Machine
  Learning}}.
\newblock
\urldef\tempurl%
\url{https://arxiv.org/abs/2501.16825}
\showURL{%
\tempurl}
\newblock
\shownote{arXiv:2501.16825}.


\bibitem[von Oswald et~al\mbox{.}(2023)]%
        {vonoswald2023transformers}
\bibfield{author}{\bibinfo{person}{Johannes von Oswald},
  \bibinfo{person}{Eyvind Niklasson}, \bibinfo{person}{Ettore Randazzo},
  \bibinfo{person}{Jo{\~a}o Sacramento}, \bibinfo{person}{Alexander
  Mordvintsev}, \bibinfo{person}{Andrey Zhmoginov}, {and} \bibinfo{person}{Max
  Vladymyrov}.} \bibinfo{year}{2023}\natexlab{}.
\newblock \showarticletitle{Transformers Learn In-Context by Gradient Descent}.
\newblock \bibinfo{journal}{\emph{arXiv preprint arXiv:2212.07677}}
  (\bibinfo{year}{2023}).
\newblock
\newblock
\shownote{Also appeared in ICML 2023}.


\bibitem[Xie et~al\mbox{.}(2022)]%
        {xie2022explanation}
\bibfield{author}{\bibinfo{person}{Sang~Michael Xie}, \bibinfo{person}{Aditi
  Raghunathan}, \bibinfo{person}{Percy Liang}, {and} \bibinfo{person}{Tengyu
  Ma}.} \bibinfo{year}{2022}\natexlab{}.
\newblock \showarticletitle{An Explanation of In-Context Learning as Implicit
  {B}ayesian Inference}. In \bibinfo{booktitle}{\emph{International Conference
  on Learning Representations}}.
\newblock


\end{thebibliography}

\end{document}